\documentclass[10pt,journal,compsoc]{IEEEtran}

\usepackage[linesnumbered,ruled,vlined]{algorithm2e}

\usepackage[table]{xcolor}
\usepackage{amsmath,amsfonts}
\usepackage{array}
\usepackage[caption=false,font=normalsize,labelfont=sf,textfont=sf]{subfig}
\usepackage{textcomp}
\usepackage{stfloats}
\usepackage{url}
\usepackage{verbatim}
\usepackage{graphicx}
\usepackage{cite}
\usepackage{booktabs}
\usepackage{multirow}
\usepackage{ragged2e}
\usepackage{amsthm} 
\usepackage{makecell}

\usepackage{ulem}
\usepackage{xcolor} 

\newtheorem{definition}{Definition}[section]

\theoremstyle{plain}
\newtheorem{theorem}{Theorem}[section]

\theoremstyle{definition}

\theoremstyle{remark}

\definecolor{lightgray}{gray}{0.9}

\usepackage{bm}
\DeclareMathAlphabet{\mathsfit}{\encodingdefault}{\sfdefault}{m}{sl}
\SetMathAlphabet{\mathsfit}{bold}{\encodingdefault}{\sfdefault}{bx}{n}

\begin{document}

\title{\textit{Co$^2$In}: a Bi-level Memory Incremental Learning Framework with Knowledge Encoding, Consolidation, and Integration}


\author{Wenju~Sun, 
Qingyong~Li,
Boyang~Li,
Wen~Wang,
Siyu~Zhang,
Yi~Wei,
Yangli-ao~Geng

\IEEEcompsocitemizethanks{

\IEEEcompsocthanksitem Wenju~Sun, Qingyong~Li, Wen~Wang, Siyu~Zhang, and Yangli-ao~Geng are with Key Laboratory of Big Data \& Artificial Intelligence in Transportation (Ministry of Education), Beijing Jiaotong University, Beijing, China, 100044. Emails: \{SunWenJu, liqy, wangwen, zhangsiyu, gengyla\}@bjtu.edu.cn\protect\\
\IEEEcompsocthanksitem Boyang Li is with College of Computing and Data Science, Nanyang Technological University, Singapore, 639798. Email: boyang.li@ntu.edu.sg\protect\\

\IEEEcompsocthanksitem Yi Wei is with Fujian Key Laboratory of Big Data Application and Intellectualization for Tea Industry, Wuyi University, Fujian, China, 354300. Email: wyqbits@sina.com\protect\\
\IEEEcompsocthanksitem Corresponding author: Yangli-ao Geng.}
}

\markboth{Journal of IEEE Transactions On Pattern Analysis and Machine Intelligence}
{Shell \MakeLowercase{\textit{et al.}}: A Sample Article Using IEEEtran.cls for IEEE Journals}


\IEEEoverridecommandlockouts
\IEEEpubid{\begin{minipage}{\textwidth}\ \\[12pt]\ \copyright 2026 IEEE.  Personal use of this material is permitted.  Permission from IEEE must be obtained for all other uses, in any current or future media, including reprinting/republishing this material for advertising or promotional purposes, creating new collective works, for resale or redistribution to servers or lists, or reuse of any copyrighted component of this work in other works.\end{minipage}}

\IEEEtitleabstractindextext{%
\begin{abstract} \justifying 
Incremental learning (IL) aims to continually acquire new knowledge (\textbf{plasticity}) while retaining previously learned information (\textbf{stability}). However, striking a balance between plasticity and stability remains a significant challenge for intelligent systems. The human brain achieves exceptional balance, owing to various memory units that collaboratively encode and store information. Inspired by human memory mechanisms, this paper introduces an IL framework with knowledge en\textbf{Co}ding, \textbf{Co}nsolidation, and \textbf{In}tegration (\textit{Co$^2$In}). \textit{Co$^2$In} is designed as a bi-level memory architecture: a working memory for adaptive knowledge acquisition and a long-term memory dedicated to the persistent retention of information. The two memory modules work cooperatively during IL. The working memory first learns from new data to encode knowledge into parameters. Subsequently, \textit{Co$^2$In} performs a consolidation process to identify underlying patterns in the learned parameters and re-express them into a compact knowledge representation. Next, the knowledge representation and the identified patterns are transformed into separate network layers and integrated into the long-term memory. These designs empower \textit{Co$^2$In} to accumulate knowledge with high plasticity and stability. We evaluate \textit{Co$^2$In} on CIFAR-10, CIFAR-100, and Tiny-ImageNet under exemplar-free Class-IL and Task-IL settings. Experimental results show that \textit{Co$^2$In} achieves state-of-the-art performance with efficient memory consumption.

\end{abstract}

}

\maketitle

\section{Introduction}

In an ever-changing environment, incremental learning (IL) emerges as a critical challenge for intelligent systems to continually accumulate new knowledge while retaining previously acquired information \cite{survey3}. In typical IL scenarios, a learner is required to learn a sequence of disjoint tasks. When learning a new task, only data from the new task are available. In this training schema, an effective IL learner should be able to flexibly learn and incorporate new information from the current task (\textbf{plasticity}) while retaining previously acquired old-task knowledge (\textbf{stability}).
To strike a balance between plasticity and stability, a variety of IL strategies have been developed based on distinct technical mechanisms\cite{survey1, survey2}. Among them, memory-based IL methods have shown superior performance \cite{der, gpm, adam-nscl, x-der}.

\begin{figure*}[!t]
    \centering
    \includegraphics[width=0.85\textwidth]{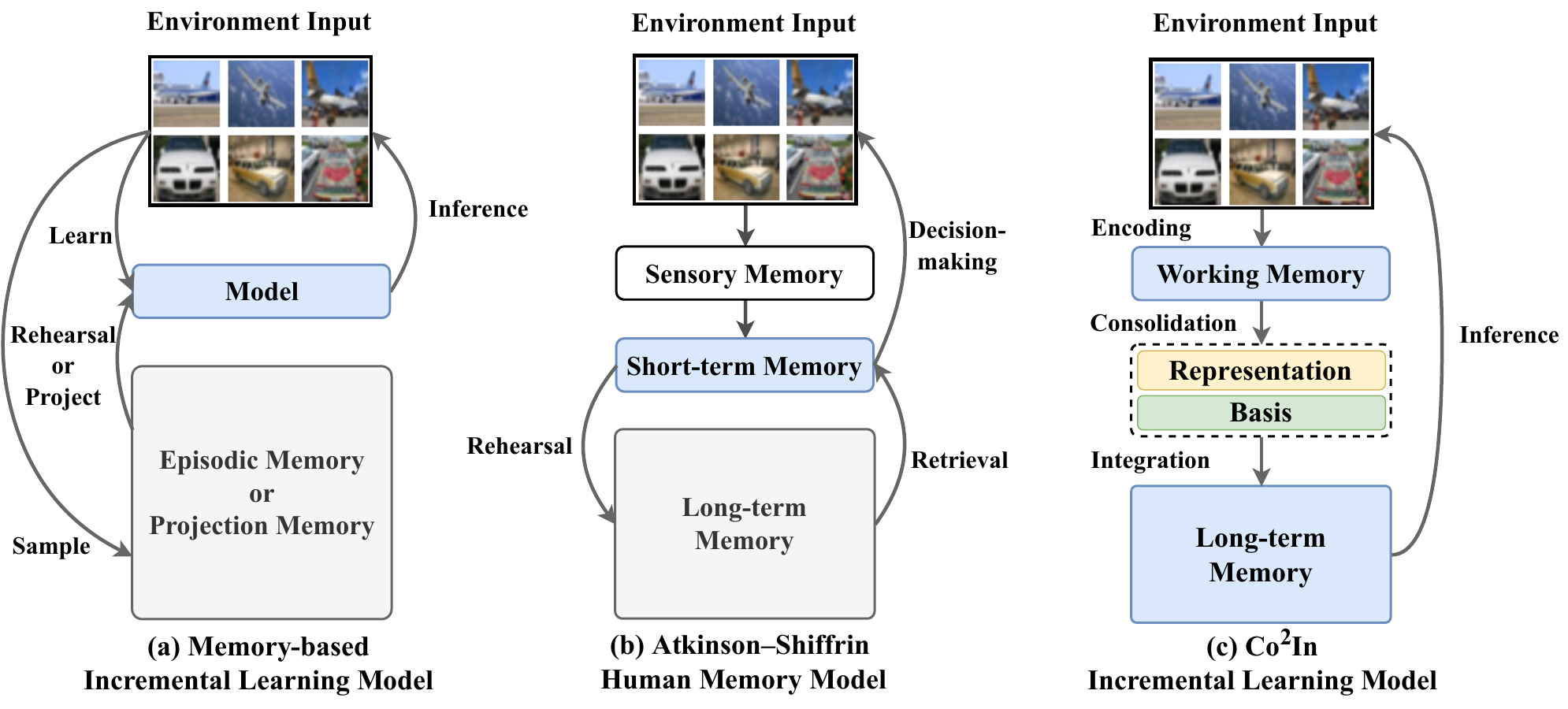}
    \caption{
        Diagram of memory-based IL methods, Atkinson-Shiffrin human memory model \cite{human_memory}, and the proposed \textit{Co$^2$In}. (a) Memory-based IL methods prevent the model from forgetting when learning by replaying old-task exemplars or restricting gradient directions. (b) In the human brain, information is temporally stored in short-term memory. After repeated rehearsal, the information will be permanently remembered in long-term memory. (c) Inspired by the human brain, \textit{Co$^2$In} acquires new knowledge with working memory and integrates it into long-term memory for storage.
    }
    \label{fig:architecture}
\end{figure*}

As illustrated in Fig.~\ref{fig:architecture} (a), memory-based methods typically involve an auxiliary memory to store partial information from previous tasks. The information is used to relieve forgetting when learning new tasks. Depending on the type of information stored, memory-based methods can be further categorized into rehearsal-based and gradient-memory-based methods. Rehearsal-based methods \cite{der, x-der, lider, gem, a-gem, bic, icarl, abd} adopt an exemplar memory (or a generative model) to store (or generate) samples or features from old tasks, replaying them to reinforce old-task knowledge during new-task learning. Alternatively, gradient-memory-based methods \cite{gpm, adam-nscl, sd, adns, rgo} choose to monitor gradient directions that might impede old-task performance, updating the model only with gradients orthogonal to the stored ones. 

Both rehearsal-based and gradient-memory-based methods adhere to a fundamental principle: the model must simultaneously acquire new knowledge while retaining old knowledge. However, knowledge acquiring and remembering typically interfere with each other, leading to the stability-plasticity dilemma \cite{dilemma}. For the rehearsal-based methods, fitting new tasks could overwrite old knowledge, particularly if the stored samples inadequately represent the data distribution of the old tasks, incurring compromised stability. On the other hand, gradient memory-based methods may hinder parameter updates along the direction of the steepest descent for loss minimization, especially in extreme cases where the entire gradient space is constrained, resulting in reduced plasticity.

In contrast, humans exhibit remarkable IL abilities, attributed to the multi-level structure of the human brain and effective memory procedures. Figure~\ref{fig:architecture} (b) illustrates the classic Atkinson-Shiffrin human memory model \cite{human_memory}. In this schema, environmental information is first captured by sensory memory and temporarily stored in short-term memory. Through repeated rehearsals, the encoded information undergoes consolidation with the assistance of the prefrontal cortex, eventually leading to its permanent storage in the long-term memory \cite{memory_consolidation, memory_construct}. The layered structure effectively segregates the tasks of information encoding and storage into different brain organs, endowing human memory with an ideal balance of plasticity and stability. 

Inspired by the mechanism of human memory, we introduce a bi-level memory framework for IL. As illustrated in Fig.~\ref{fig:architecture}(c), the framework consists of a working memory (akin to the short-term memory in the human memory model) and a long-term memory, each implemented as neural networks but with distinct roles. The working memory dynamically acquires new knowledge, while the long-term memory consistently stores all learned information. This cognitively inspired framework decouples learning from remembering, ensuring sufficient plasticity and stability. 


Based on the bi-level memory framework, we propose the knowledge en\textbf{Co}ding, \textbf{Co}nsolidation, and \textbf{In}tegration (\textbf{\textit{Co$^2$In}}) strategy to continually accumulate knowledge. Upon encountering a new task, the working memory first encodes the new knowledge into trainable parameters. Acknowledging the inherent over-parameterization issue within these trained parameters \cite{compression_survey}, \textit{Co$^2$In} executes a consolidation process to identify the underlying patterns of the trained parameters. The identified patterns serve as a set of basis to construct a parameter space that encapsulates valuable knowledge but with significantly lower dimensions. By projecting the trained parameters with respect to the pattern basis, \textit{Co$^2$In} generates a compact representation of valuable knowledge involved in the trained parameters. Then, the long-term memory re-organizes the knowledge representations and pattern basis of all learned tasks into independent layers, ultimately integrating them into a unified network. The unified network directly executes inference operations. Moreover, we equip \textit{Co$^2$In} with a novel quadratic classifier, which establishes compact decision boundaries. As a result, \textit{Co$^2$In} contains cross-task discriminability without necessitating the task identification of the query sample, thus enhancing its generalizability in task-agnostic scenarios \cite{setting}.

The contributions of this paper are listed as follows:
\begin{itemize}
  \item Inspired by the human memory system, we propose \textit{Co$^2$In}, a bi-level memory IL method that decouples knowledge acquisition and retention to balance plasticity and stability. A series of strategies are devised for knowledge encoding, consolidation, and integration within \textit{Co$^2$In}, enabling effective IL with high memory and computational efficiency.

  \item We propose a quadratic classifier to encourage \textit{Co$^2$In} to learn compact decision boundaries, which further extends its capability for task-agnostic scenarios, so \textit{Co$^2$In} is applicable to both Task-IL and Class-IL.
  
  \item \textit{Co$^2$In} is evaluated on CIFAR-10, CIFAR-100, and Tiny-ImageNet datasets. Our experimental results demonstrate that \textit{Co$^2$In} outperforms most state-of-the-art baselines with lower memory usage. 
\end{itemize}

A preliminary conference version of this work, referred to as BMKP, has been presented in \cite{bmkp}. This paper expands upon the conference version with several key enhancements: (1) We introduce a more efficient knowledge integration process, which markedly reduces computational overhead during inference. (2) A quadratic classifier is incorporated into the bi-level memory framework, notably improving the performance of \textit{Co$^2$In} in exemplar-free Class-IL, an aspect not covered in BMKP. (3) We offer more experiments to comprehensively validate \textit{Co$^2$In}.


\section{Related Work} \label{sec:related_work}

\subsection{Incremental Learning}

According to the techniques for balancing stability and plasticity, current IL methods can be classified into three categories: memory-based, regularization-based, and architecture-based. 

\textbf{Memory-based methods} typically preserve partial previous tasks information to mitigate catastrophic forgetting. According to the stored information, memory-based methods can be further split into rehearsal-based methods and gradient-memory based methods. Some rehearsal-based methods \cite{der, x-der, icarl, hcr, co2l, clser, lider, onpro} store old-task samples as exemplars, and replay them alongside new data during learning. Other rehearsal-based methods \cite{dgr, abd, di, PGMA, dcpoc} maintain a generative model capable of producing pseudo-samples that mimic the distribution of old tasks, also known as generative replay. On the other hand, gradient-memory based methods \cite{gem, a-gem, gpm, adam, sd, rgo} adopt a different strategy by preserving the directions of gradients that could disrupt the performance on earlier tasks. These recorded gradients serve as constraints during model updates: only those gradient directions orthogonal to them are permitted for updating the parameters, thereby preserving previously acquired knowledge. Although \textit{Co$^2$In} employs a bi-level memory structure, it avoids exemplar storage and gradient constraints by instead retaining low-rank knowledge patterns and representations.

\textbf{Regularization-based methods} alleviate forgetting by incorporating regularization terms into the loss function. Learning without Forgetting (LwF) \cite{lwf} first applies knowledge distillation to regularize the model, aiming to align the current model with a teacher model representing the old-task knowledge. As knowledge distillation is effective in reducing forgetting, this technique has become a standard IL baseline and has been widely adopted in many IL methods \cite{lwf, podnet, tpcil, vdfd, praka}. Another prominent regularization-based method is Elastic Weight Consolidation (EWC) \cite{ewc}. EWC leverages the Fisher information matrix to estimate the importance of each parameter and punishes their changes during subsequent task training. Several IL methods follow a similar idea to alleviate forgetting, such as Synaptic Intelligence (SI) \cite{si}, oEWC \cite{oewc}, and Memory Aware Synapses (MAS) \cite{mas}. In this paper, we also propose a regularization term for \textit{Co$^2$In} to encourage the model to reuse previously acquired knowledge patterns in new tasks.

\textbf{Architecture-based methods} address forgetting by employing specifically designed network structures. Some mask-based methods \cite{packnet, cpg, piggyback, spg} learn a set of parameter masks that indicate the assignment of neurons. These masks facilitate the selective activation or deactivation of specific parts of the network for each task. Other expansion-based methods \cite{pnn, dcpoc, ilcoc, den, iru} expand the width of networks to ensure sufficient network capacity for new-task learning. During inference, most architecture-based methods require explicit task identification of queries to adapt the parameters accordingly, which limits their application in task-agnostic scenarios, e.g., Class-IL. To address this issue, some architecture-based methods leverage one-class classifiers \cite{ilcoc, pcl, dcpoc} or confidence \cite{eft} to achieve task-wise discrimination. Our \textit{Co$^2$In} increases the network capacity by dynamically expanding pattern basis. Moreover, similar to one-class classifier based methods, it learns compact decision boundaries using a quadratic classifier, enabling task-agnostic inference without explicit task identifiers.

Additionally, various IL methods are inspired by the human cognitive process. Drawing inspiration from the phenomenon of synaptic expansion (neurons expand new synapses to acquire new information), Wang et al. \cite{afec} propose to expand new network branches when learning a new task. The parameters of the newly expanded branch are selectively merged with the original branch through knowledge distillation. Sarfraz et al. \cite{scommer} suggest a multi-memory system to mimic the human memory structure, including episodic memory, working memory, and long-term memory. \textit{Co$^2$In} shares a similar motivation with these approaches but with different underlying mechanisms. In \textit{Co$^2$In}, all knowledge is stored in the long-term memory, while the working memory is solely responsible for acquiring knowledge. With this design, \textit{Co$^2$In} decouples the roles of learning and remembering into different memory units, thereby achieving both high plasticity and stability. In contrast, the responsibilities of knowledge acquiring and remembering in existing cognitive-inspired IL methods are still concentrated within the same model.

Some recent IL methods are designed for pretrained models, which benefit from a well-initialized, transferable representation space. Prompt-based methods insert learnable prompts into frozen backbones, where L2P \cite{pre_1} dynamically selects input-dependent prompts, DualPrompt \cite{pre_2} disentangles task-shared and task-specific prompts across layers, and CODA-Prompt \cite{pre_3} further enables compositional prompt generation. Adapter-based methods introduce lightweight modules for task adaptation, including EASE \cite{pre_5}, which learns task-specific adapters via prototype completion to mitigate feature–classifier mismatch; SSIAT \cite{pre_6}, which maintains a shared adapter and performs unified classifier retraining; and ACMAP \cite{acmap}, which improves scalability by merging multiple adapters. Other LoRA-based approaches constrain updates to low-rank subspaces, with InfoLoRA \cite{infolora} constructing interference-aware subspaces and CL-LoRA \cite{pre_9} combining shared and task-specific LoRA modules. In this paper, the proposed \textit{Co$^2$In} focuses on the training-from-scratch setting, while it can be naturally extended to pretrained-based IL by learning a low-rank residual branch, enabling seamless integration with parameter-efficient adaptation paradigms.

\subsection{Human Memory Mechanism}

Human beings exhibit remarkable plasticity and stability throughout their lives, owing to the brain's multi-layered memory structure and effective memory mechanisms. The classical Atkinson-Shiffrin model \cite{human_memory} posits that human memory includes three main units: sensory memory, short-term memory, and long-term memory. Sensory memory acts as a cache for sensory activations like sounds or touches, transferring information to short-term memory when attention is concentrated. Short-term memory holds information for immediate use, and after repeated rehearsal, this information is consolidated into long-term memory for permanent storage. This hierarchical structure ensures both the plasticity and stability of the brain.

This paper draws inspiration from the human brain's memory mechanisms. Leveraging the hierarchical structure of multi-level human memory \cite{human_memory}, we propose a bi-level memory IL framework that separates knowledge acquisition and retention into specialized memory units, ensuring both plasticity and stability. guided by the cognitive mechanisms of the human memory process  \cite{memory_construct}, we propose the knowledge encoding, consolidation, and integration strategy within the bi-level memory framework. This strategy ensures both memory and computational efficiency.

\begin{table*}[htbp]
\centering
\caption{Mathematical notations and their descriptions.}
\label{tab:symbols}
\renewcommand{\arraystretch}{0.6} 
\begin{tabular}{llp{10cm}}
\toprule
\textbf{Notation} & \textbf{Description} \\
\midrule

\multicolumn{2}{c}{\textit{\textbf{General Notations}}} \\
$\left\| . \right\|_1$ & $\ell_1$ norm: sum of absolute values of all elements. \\
$\left\| . \right\|_F$ & Frobenius norm: $\sqrt{\sum_{i,j} |a_{ij}|^2}$ for a matrix $\bm{A}$. \\
$\text{SVD}(.)$ & Singular Value Decomposition: $\text{SVD}(\bm{X}) = \bm{U}\bm{\Sigma}\bm{V}^\top$, where $\bm{U}$ and $\bm{V}$ are orthogonal, and $\bm{\Sigma}$ is diagonal. \\
$L_{ce}(.)$ & The cross-entropy loss. \\
$[.,.]$ & The concatenation operation. \\

\midrule
\multicolumn{2}{c}{\textit{\textbf{Dataset and Layer-wise Features Transformation Notations}}} \\
$\mathcal{D}_t$ & Dataset of task $t$: $\mathcal{D}_t = \{(\bm{x}_k, y_k)\}_{k=1}^{n_t}$, where $n_t$ is the sample size. \\
$\bm{X}^l \in \mathbb{R}^{p^l \times n}$ & Input features of the $l$-th layer: $\bm{X}^l = [\bm{x}_1^l, \dots, \bm{x}_k^l, \dots, \bm{x}_{n}^l]$, where $\bm{x}_k^l \in \mathbb{R}^{p^l}$ corresponds the feature for the $k$-th sample. \\
$\bm{Z}^l \in \mathbb{R}^{p^{l+1} \times n}$ & Transformed features of the $l$-th layer: $\bm{Z}^l = \bm{W}^l \bm{X}^l$. \\
$\bm{W}^l \in \mathbb{R}^{p^{l+1} \times p^l}$ & Parameters of the $l$-th layer. \\
$\sigma$ & The parameter-free non-linear unit such as the ReLU activation: $\bm{X}^{l+1} = \sigma(\bm{Z}^l)$. \\

\midrule
\multicolumn{2}{c}{\textit{\textbf{\textit{Co$^2$In} Specific Notations: Knowledge Representation and Pattern Basis}}} \\
$\mathcal{A}$ & \makecell[l]{Knowledge representations of all tasks: $\mathcal{A} = \{\mathcal{A}_1, \mathcal{A}_2, \dots\}$, where $\mathcal{A}_t = \{\bm{A}_t^1, \bm{A}_t^2, ..., \bm{A}_t^L\}$ is the knowledge\\ representation specific to task $t$}. \\
$\bm{A}_t^l \in \mathbb{R}^{r_t^l \times p^l}$ & Knowledge representation of the $l$-th layer for task $t$. \\
$\mathcal{B}_t$ & $\mathcal{B}_t = \{\bm{B}_t^1, ..., \bm{B}_t^L\}$ represents the pattern basis of all layers after task $t$. \\
$\bm{B}_t^l \in \mathbb{R}^{p^{l+1} \times r_t^l}$ & Pattern basis matrix of the $l$-th layer after learning task $t$. \\
$r_t^l$ & Number of retained basis for the $l$-th layer after task $t$. \\
$\hat{\bm{W}}^l \in \mathbb{R}^{p^{l+1} \times p^l}$ & The projection of $\bm{W}^l$ onto the basis of $\bm{B}^l$: $\hat{\bm{W}}^l = \bm{B}^l{\bm{B}^l}^\top \bm{W}^l$. \\
$\tilde{\bm{W}^l}  \in \mathbb{R}^{p^{l+1} \times p^l}$ & The residual of the orthogonal projection of $\bm{W}$ onto CKS: $\tilde{\bm{W}^l} = \bm{W}^l- \bm{B}^l{\bm{B}^l}^\top \bm{W}^l$. \\

\bottomrule
\end{tabular}
\end{table*}

\section{Problem Definition} \label{sec:problem_definition}

Currently, two IL settings are drawing attention from the research community, namely task incremental learning (Task-IL) and class incremental learning (Class-IL) \cite{setting} :

\begin{definition}[Task-IL]
Consider a sequence of $T$ classification tasks $\{ \mathcal{D}_1, \dots, \mathcal{D}_T \}$, each associated with a unique dataset $\mathcal{D}_t = \{(\bm{x}_k, y_k)\}_{k=1}^{n_t}$, where $\bm{x}_k$ and $\bm{y}_k \in \mathcal{Y}_t$ represents a input-label pair. It is critical to note that the classes for classification in different tasks are mutually exclusive, i.e., $\mathcal{Y}_i \cap \mathcal{Y}_j = \emptyset, \forall i \neq j$. The goal of the model is to incrementally learn these tasks, without the need to refer back to the data of previous tasks once they have been learned. During the test stage of Task-IL, each query sample is accompanied by its task identification to assist in classification.
\end{definition}

\begin{definition}[Class-IL]
    Class-IL shares the same setup as Task-IL, except during the testing stage, where task identification is no longer available for each query sample.
\end{definition}

The proposed \textit{Co$^2$In} is feasible for both Task-IL and Class-IL. Moreover, considering the application in privacy-sensitive scenarios, this paper focuses on the two settings under the challenging \textbf{exemplar-free} restriction \cite{classilsurvey}, where no samples from previous tasks are retained for replaying.

To achieve this, we investigate layer-wise operations of neural networks, where features are progressively transformed across layers. By decomposing and preserving parameters layer-by-layer, the model achieves efficient IL while mitigating forgetting. Relevant notations for these operations are summarized in Table~\ref{tab:symbols}, organized into the following categories:

\begin{itemize}
    \item General notations: commonly used mathematical operations and norms, such as $\ell_1$ norm ($\left\| . \right\|_1$), Frobenius norm ($\left\| . \right\|_F$), singular value decomposition (SVD), and cross-entropy loss ($L_{ce}(.)$).
    \item Dataset and layer-wise feature transformations, including input features ($\bm{X}^l$), transformed features ($\bm{Z}^l$), parameters ($\bm{W}^l$), and non-linear activation ($\sigma$).
    \item Notations specific for \textit{Co$^2$In}: knowledge representations across tasks ($\mathcal{A}$), pattern basis ($\mathcal{B}_t$), and projection-related terms such as $\hat{\bm{W}}^l$ (projection of weights) and $\tilde{\bm{W}}^l$ (residual projection).
\end{itemize}

\begin{figure*}[!t]
    \centering
    \includegraphics[width=1.0\textwidth]{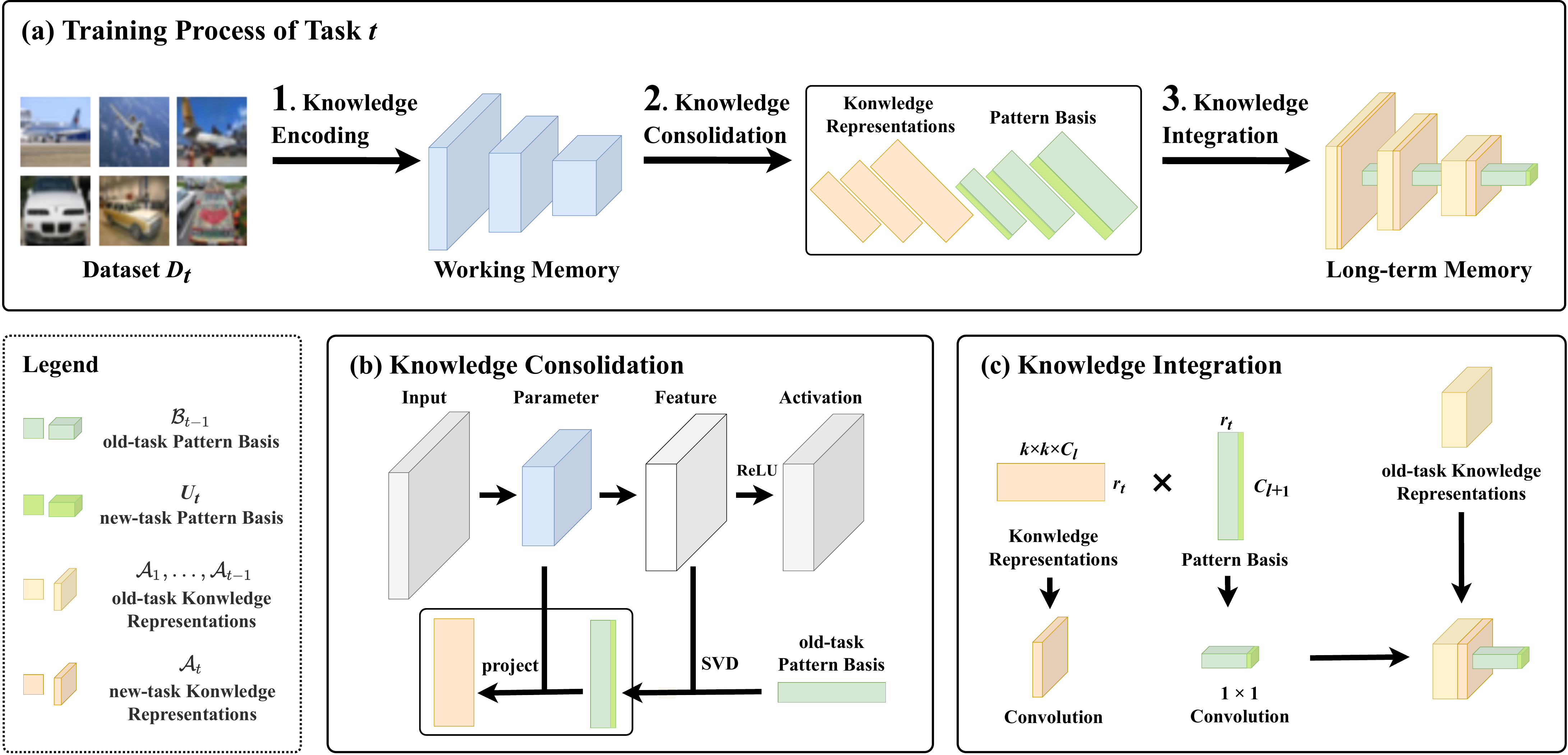}
    \caption{
        Illustration of the learning process of \textit{Co$^2$In}.
        (a) The task-$t$ training process of \textit{Co$^2$In}, including knowledge encoding, knowledge consolidation, and knowledge integration steps.
        (b) In the knowledge consolidation step, \textit{Co$^2$In} recognizes new pattern basis and projects the parameters to form compact knowledge representations. (c) In the knowledge integration step, \textit{Co$^2$In} reorganizes the long-term memory into a unified neural network.
    }
    \label{fig:process}
\end{figure*}

\section{Bi-level Memory Framework}\label{sec:framework}

A naive IL solution is to train an independent model for each task, but storing and running all task-specific models lead to increasing memory and computational overhead. Our proposed approach involves using low-rank matrices to compactly represent incremental knowledge learned from each additional task and gradually expand the network width. This forms the bi-level structure introduced in this section, a novel IL framework that separates knowledge acquisition and retention into two distinct memory units. Before describing the bi-level structure, we first elaborate on two knowledge spaces that serve as a guide for organizing knowledge within these memory units.

\subsection{Knowledge Spaces}
\label{ssec:know_spac}

In the context of neural networks, we define \textbf{knowledge} as the inherent capability to map a specific input into the corresponding output. The network parameters can be thought of as an explicit representation of knowledge, and we define the space which resides in as \textbf{Parameter Knowledge Space (PKS)}. PKS often exhibits redundancy in dimensions, as current deep neural models are typically overparameterized \cite{compression_survey}. In light of this, we introduce \textbf{Core Knowledge Space (CKS)} as a subspace of PKS that contains all meaningful information but with much fewer dimensions.

We pursue a CKS by investigating the mechanism of neural networks. Consider an $L$-layered neural network with parameters $\mathcal{W}=\{\bm{W}^1,...,\bm{W}^L\}$, where $\bm{W}^l$ is a $p^{l+1}$-by-$p^{l}$ matrix representing the parameters of the $l$-th layer. The network receives input features and executes the following feature transformation layer by layer:
\begin{equation}
  \bm{Z}^l = \bm{W}^l\bm{X}^l \ \ \ \ \bm{X}^{l+1} = \sigma(\bm{Z}^l),
  \label{equation:layer_func} 
\end{equation}
where $\bm{X}^l$ is the input of layer $l$ and $\bm{X}^{l+1}$ is the output of layer $l$.
It is worth emphasizing that, without loss of generality, we conceptualize the parameters $\bm{W}^l$ as a matrix, given the inherent linearity within current network layers. Section~\ref{sec:cnn} details the extension to convolutional layers.

Equation~\eqref{equation:layer_func} evidence that the feature $\bm{Z}^l$ is a linear combination of the columns of $\bm{W}^l$. This observation indicates that only those columns of $\bm{W}^l$ that contribute to the combination of $\bm{Z}^l$ contain valuable knowledge, i.e., parts of $\bm{W}^l$ that reside within the column space of $\bm{Z}^l$.
To elucidate further, suppose that the columns of a $p^{l+1}$-by-$r^{l}$ ($r^{l} \ll  p^{l+1}$) matrix $\bm{B}^l$ form an orthogonal basis for the space spanned by $\bm{Z}^l$, the first term in Eq.~\eqref{equation:layer_func} can be re-expressed as: 
\begin{equation*}
  \begin{aligned}
    \bm{Z}^l=\bm{B}^l{\bm{B}^l}^\top \bm{Z}^l & = \bm{B}^l{\bm{B}^l}^\top \bm{W}^l \bm{X}^l \\
     & = \hat{\bm{W}}^l\bm{X}^l\\
     & = \bm{B}^l\bm{A}^l\bm{X}^l, \\
  \end{aligned}
\end{equation*}
which forms a linear space equivalent to the column space of $\bm{Z}^l$: $\hat{\bm{W}}^l$ transforms the given input $\bm{X}^l$ to the expected output $\bm{Z}^l$, implying that $\hat{\bm{W}}^l$ encapsulates the most knowledge in $\bm{W}^l$. $\bm{A}^l = {\bm{B}^l}^\top \bm{W}^l$ stands for the coefficient matrix of $\hat{\bm{W}}^l$ under the basis of $\bm{B}^l$, which can be interpreted as a compact knowledge representation of $\bm{W}^l$ with respect to the basis $\bm{B}^l$. 
We conceptualize the space spanned by $\bm{Z}^l$ as CKS and refer $\bm{B}^l$ as pattern basis. One possible way to find such a $\bm{B}^l$ is to perform singular value decomposition (SVD) of $\text{SVD}(\bm{Z}^l)=\bm{U}^l\bm{\Sigma}^l{\bm{V}^l}^\top$ and let $\bm{B}^l = \bm{U}^l$.

\subsection{Bi-level Memory Structure}  \label{sec:bilevel}

We propose a bi-level memory framework based on PKS and CKS. As depicted in Fig.~\ref{fig:architecture} (c), The framework consists of the following two memory units:

\begin{itemize}
\item \textbf{Working memory} is responsible for adaptively acquiring new knowledge in PKS. It is implemented as a $L$-layer neural network, with parameters denoted as $\mathcal{W}=\{\bm{W}^1,...,\bm{W}^L\}$.

\item \textbf{Long-term memory} dynamically maintains an entire CKS during incremental learning, including a set of pattern basis $\mathcal{B} = \{\bm{B}^1,...,\bm{B}^L\}$ for each layer, alongside the projected knowledge representations for all learned tasks $\mathcal{A} = \{\mathcal{A}_1, ..., \mathcal{A}_T\}$.
\end{itemize}

The bi-level memory framework decouples the roles of knowledge acquiring and remembering into specialized memory units, thereby ensuring high plasticity and stability.

\section{Knowledge Encoding, Consolidation, and Integration}\label{sec:co2in}

Based on the bi-level memory framework, this section outlines an IL strategy with knowledge encoding, consolidation, and integration (\textit{Co$^2$In}). As illustrated in Fig.~\ref{fig:process}, during the knowledge encoding phase, the working memory fits with new-task data to encode knowledge into parameters. In the knowledge consolidation stage, \textit{Co$^2$In} projects the parameter-formed knowledge onto CKS to constitute a much more compact knowledge representation. Subsequently, in the knowledge integration step, the long-term memory re-organizes all knowledge representations and pattern basis into a unified neural network. 

We will elaborate on knowledge encoding, consolidation, and integration steps in Section~\ref{sec:knowledge_encoding}, \ref{sec:knowledge_consolidation}, and \ref{sec:knowledge_integration}, respectively. Additionally, in Section~\ref{sec:cnn}, we present a straightforward strategy for extending our methods to convolutional neural networks. Lastly, in Section~\ref{sec:quadratic_classifier}, we propose a quadratic classifier that enhances the Class-IL ability of \textit{Co$^2$In} with compact decision boundaries. 

\subsection{Knowledge Encoding} \label{sec:knowledge_encoding}

Knowledge encoding aims to encapsulate new knowledge from data into the working-memory parameter. 

During training, we incorporate a representation compaction regularizer to ensure that the working memory aligns with the existing pattern basis, formulated as follows:
\begin{equation*}
  L_{reg}(\mathcal{W})=\sum_{\bm{W}^l \in \mathcal{W}} \frac{
  \left\|    \tilde{\bm{W}}^l         \right\|_F^2
  }{
  \left\|   { \tilde{\bm{W}^l} }^\top  \tilde{\bm{W}^l} \right\|_1
  },
  \label{equation:reg} 
\end{equation*}
where $\tilde{\bm{W}^l} = \bm{W}^l- \bm{B}_{t-1}^l\bm{B}_{t-1}^l{}^\top \bm{W}^l$ signifies the residual of the orthogonal projection of $\bm{W}^l$ onto CKS. This regularizer serves two pivotal roles in enhancing performance and efficiency. Firstly, the numerator of Eq.~\eqref{equation:reg} penalizes the Frobenius norm of the projection residual, which encourages knowledge encoding by leveraging previously learned patterns $\bm{B}_{t-1}^l$, thereby facilitating the transfer of old-task knowledge to current task learning. Secondly, the maximization of the $\ell_1$ vector norm in the denominator encourages the columns of the residual to be co-linear:
\begin{equation*}
  \begin{aligned}
  \left\| \tilde{\bm{W}}^l{}^\top \tilde{\bm{W}}^l \right\|_1 = \sum_{i,j}  | \tilde{\bm{w}}^l_i {}^\top  {\tilde{\bm{w}}^l_j}  |,
  \end{aligned}
\end{equation*}
thereby regularizing $\tilde{\bm{W}^l}$ to be approximately low-rank (i.e., the rows approximately lie in a low-dimensional subspace). Empirical evidence in Fig.~\ref{fig:sensitive} has shown that with an appropriate weight, the regularizer significantly reduces memory overhead while improving the performance of \textit{Co$^2$In}.  

The overall loss function for knowledge encoding is
\begin{equation}
  \mathcal{W}_t \longleftarrow \mathop{\arg\min}\limits_{\mathcal{W}} L_{ce}(\mathcal{W}, \mathcal{D}_t) + \lambda L_{reg}(\mathcal{W}),
  \label{equation:loss} 
\end{equation}
where $\lambda$ is a hyper-parameter weighting the regularizer.

\subsection{Knowledge Consolidation} \label{sec:knowledge_consolidation}

\textit{Co$^2$In} next executes knowledge consolidation to project new knowledge from PKS onto CKS.
To minimize the loss of information during projection, we initially recognize new pattern basis based on the column space of $\bm{Z}_t^l$. Instead of directly performing SVD on $\bm{Z}^l_t = \bm{U}^l_{t}\bm{\Sigma}^l_{t} {\bm{V}^l_{t}}^\top$ and letting $\bm{B}_{t}^l=\bm{U}^l_{t}$, we perform SVD on the incremental knowledge learned from task $t$, represented by $\bm{Z}^l_t - {\bm{B}_{t-1}^l}{\bm{B}_{t-1}^l}^\top \bm{Z}^l_t$:
\begin{equation}
  \label{equation:new_pattern} 
  \begin{aligned}
    \bm{U}^l_{t}\bm{\Sigma}^l_{t} {\bm{V}^l_{t}}^\top&\longleftarrow \text{SVD}\left(\bm{Z}^l_t - {\bm{B}_{t-1}^l}{\bm{B}_{t-1}^l}^\top \bm{Z}^l_t\right)\\
    \bm{B}_{t}^l&\longleftarrow [\bm{B}_{t-1}^l, \bm{U}^l_{t}].
  \end{aligned}
\end{equation}

It is worth noting that $\bm{U}^l_{t}$ is guaranteed to be orthogonal to the existing basis $\bm{B}_{t-1}^l$. The orthogonality stems from the operation of projecting $\bm{Z}^l_t$ onto the left null space of $\bm{B}_{t-1}^l$, which is detailed in Theorem~\ref{the:orth}. Moreover, to keep the basis compact, we conduct basis selection with an introduced retained singular-value ratio\footnote{It is defined as the ratio of the sum of retained singular values to the sum of total singular values.}. Only the singular vector corresponding to retained values is picked as a basis.

Based on $\bm{B}_t^l$, \textit{Co$^2$In} then projects $\bm{W}_t^l$ onto CKS :
\begin{equation*}
  \bm{A}_t^l \longleftarrow {\bm{B}_t^l}^\top \bm{W}_t^l.
\end{equation*}
By applying this projection step to all layers, we obtain the updated pattern basis $\mathcal{B}_t=\{\bm{B}_t^1, \dots,\bm{B}_t^L\}$ and the knowledge representation $\mathcal{A}_t=\{\bm{A}_t^1,\dots,\bm{A}_t^L\}$. 

The selection of singular vectors may exclude some valuable patterns associated with minor singular values, potentially resulting in a loss of knowledge. To mitigate this issue, we propose a refinement strategy: finetuning $\mathcal{A}_t$ with $\mathcal{D}_t$ in CKS, which can be formulated as follows:
\begin{equation}
  \mathcal{A}_t \longleftarrow \mathop{\arg\min}\limits_{\mathcal{A}_t} L_{ce}(\mathcal{B}_t\mathcal{A}_t, \mathcal{D}_t),
  \label{equation:rehearsal} 
\end{equation}
where $\mathcal{B}_t\mathcal{A}_t = \{\bm{B}_t^1\bm{A}_t^1,...,\bm{B}_t^L\bm{A}_t^L\}$ denotes the layer-by-layer multiplication of pattern basis $\mathcal{B}_t$ and knowledge representation $\mathcal{A}_t$.
Upon completion of the finetuning, we transmit $\mathcal{A}_t$ and $\mathcal{B}_t$ into the long-term memory for the subsequent knowledge integration.

\begin{figure}[t]
    \centering
    \includegraphics[width=1.0\columnwidth]{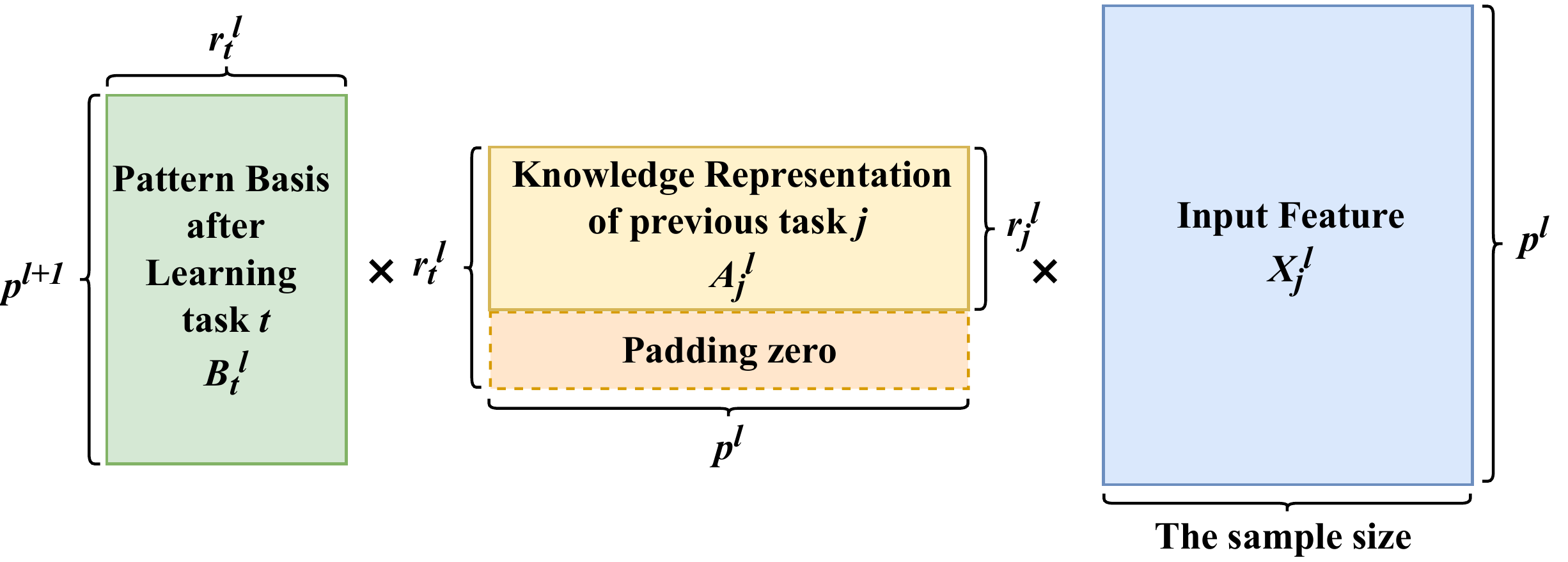}
    \caption{
        Diagram of the zero-padding operation during multiplying the task-$t$ pattern basis $\bm{B}_t^l\in\mathbb{R}^{p^{l+1}\times r_t^l}$ with previously learned task-$j$ ($j<t$) knowledge representation $\bm{A}_j^l\in\mathbb{R}^{r_j^l \times p^{l}}$ in Eq.~\eqref{equation:change2}. 
    }
    \label{fig:padding}
\end{figure}

\begin{figure}[t]
    \centering
    \includegraphics[width=1.0\columnwidth]{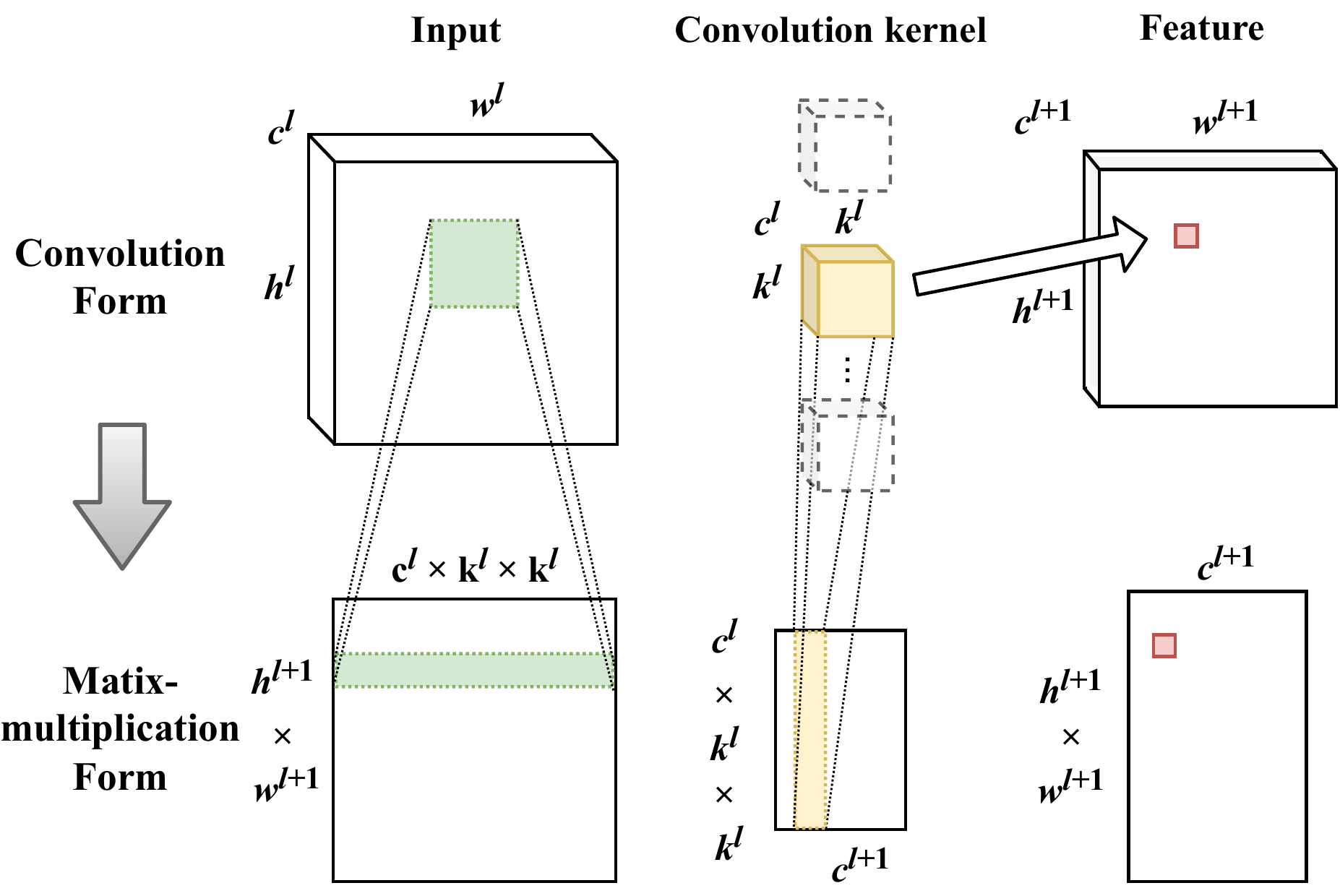}
    \caption{
        Diagram of converting a convolutional layer into a linear layer. 
        $w^l$, $h^l$, and $c^l$ denote the width, height, and channel of the input tensor, respectively; $w^{l+1}$, $h^{l+1}$, and $c^{l+1}$ correspond to the width, height, and channel of the output feature tensor, respectively, and $k^l$ signifies the size of the convolution kernel.
    }
    \label{fig:cnntofc}
\end{figure}
\begin{figure*}[t]
    \centering
    \includegraphics[width=1.0\textwidth]{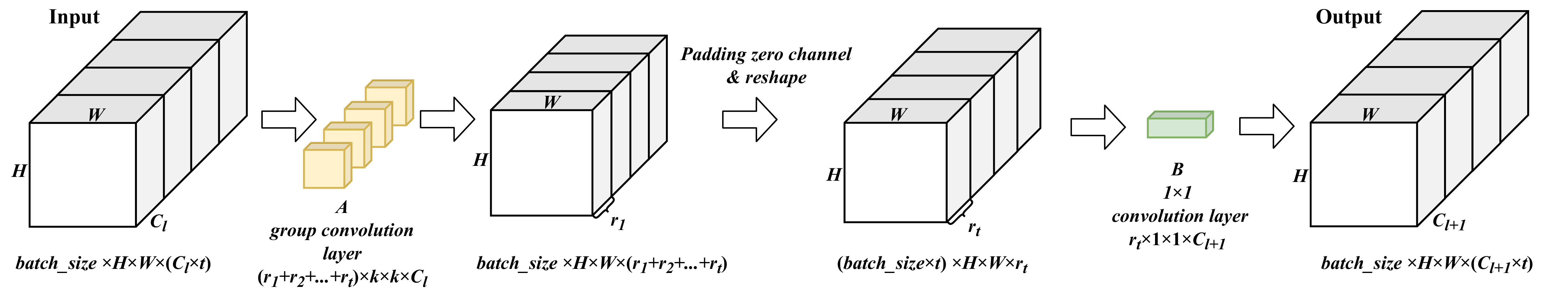}
    \caption{
        Diagram of the $l$-th convolutional layer of long-term memory after knowledge integration, which turns into a group convolutional layer (corresponding to knowledge representations) and a $1\times 1$ convolutional layer (corresponding to pattern basis). Between them, we supplement an additional zero-channel padding operation to match the dimension.
    }
    \label{fig:knowledge_integration}
\end{figure*}

\subsection{Knowledge Integration} \label{sec:knowledge_integration}

Next, \textit{Co$^2$In} integrates all acquired knowledge to enable efficient inference. The key idea is to rearrange the multiplication order in line with the associative law:
\begin{equation}
  \begin{aligned}
    \bm{Z}_t^l \!=\!\! \hat{\bm{W}_t}^l \! \bm{X}^l \!\!=\! (\bm{B}_t^l\bm{A}_t^l)\bm{X}^l \!\!=\! \bm{B}_t^l(\bm{A}_t^l\bm{X}^l)  \ \ \ \bm{X}^{l+1} \!=\! \sigma(\bm{Z}^l).
  \end{aligned}
  \label{equation:change} 
\end{equation}
Equation~\eqref{equation:change} provides an opportunity to integrate all knowledge representations into a separate layer:
\begin{equation}
\begin{aligned}
&[\bm{Z}_1^l,\dots ,\bm{Z}_t^l] = \bm{B}_t^l [\bm{A}_1^l \bm{X}_1^l,\dots ,\bm{A}_t^l \bm{X}_t^l] \\
&[\bm{X}_1^{l+1}, \dots, \bm{X}_t^{l+1}] = \sigma \left( [\bm{Z}_1^l,\dots ,\bm{Z}_t^l] \right).
\end{aligned}
\label{equation:change2} 
\end{equation}
Here, for simplicity, we omit the zero-padding operation on $\bm{A}_j^l$ when $j < t$. This operation pads zeros for the superfluous dimensions in $\bm{A}^l_j$ (corresponding to the basis of $\bm{B}_t^l$ that are never used by $\bm{A}^l_j$), as depicted in Fig.~\ref{fig:padding}.

Knowledge integration facilitates parallel inference across tasks.
Given a set of query samples $\bm{X}$, \textit{Co$^2$In} begins by duplicating the samples $t$ times, yielding inputs $[\bm{X}_1^1, ..., \bm{X}_t^1]$ corresponding $t$ tasks at the first layer. Each layer then transforms the feature according to Eq.~\eqref{equation:change2}, eventually generating the classification logits $[\bm{Z}_1^L,...,\bm{Z}_t^L]$. The class associated with the highest logit value is selected as the final classification result:
\begin{equation}
  \begin{aligned}
    \hat{\bm{Y}} = \arg \max [\bm{Z}_1^L,...,\bm{Z}_t^L].
  \end{aligned}
  \label{equation:classification} 
\end{equation}
Note that this parallel inference is necessary only for Class-IL or when a query batch contains data from multiple tasks. In typical Task-IL settings, where a query batch originates from a single task and the task ID is available, it suffices to compute only the logits associated with that specific task (i.e., $\bm{Z}_t^L$ for the query task $t$).

Knowledge integration significantly enhances computational efficiency. Consider a straightforward serial inference that directly re-composes each task parameter. Upon receiving query samples $\bm{X}$, the serial approach first reconstructs the task-1 parameters via $\mathcal{B}_t\mathcal{A}_1$, then loads them into working memory for predicting logits for task-1 classes $\bm{Z}^L_1 = f(\bm{X}; \mathcal{B}_t\mathcal{A}_1)$. This prediction process is then repeated iteratively for other task's knowledge representation, yielding logits $\{\bm{Z}^L_1, ..., \bm{Z}^L_t\}$. Notably, due to memory constraints, it is impractical to recompose all task parameters $\{ \mathcal{B}_t\mathcal{A}_1, ..., \mathcal{B}_t\mathcal{A}_t \}$ simultaneously. Assuming $r_t^{l}$ signifies the number of basis after learning task $t$, the computational complexity of serial prediction for the $l$-th layer within task $t$ approximates $\mathcal{O}(p^{l+1} p^{l} n + p^{l+1}  r_t^{l}  p^{l}) = \mathcal{O}((r_t^{l} + n) {p^{l}}^2 )$, given that $p^{l} \approx  p^{l+1}$. Here, $p^l$ represents the input dimension for the $l$-th layer, and $n$ is the number of samples. This serial prediction is time-consuming, especially for deeper neural network layers with high dimensionality (i.e., large $p^l$).

In contrast, parallel inference brought by knowledge integration is markedly more efficient. On one hand, the computational complexity in Eq.~\eqref{equation:change} reduces to $\mathcal{O}(r_t^{l} p^{l} n + p^{l+1} r_t^{l} n) = \mathcal{O}(r_t^{l} n p^{l})$, which is significantly lower than that of the serial inference ($\mathcal{O}((r_t^{l} + n) {p^{l}}^2)$). On the other hand, Eq.~\eqref{equation:change2} involves concatenated matrix multiplication, benefiting more from GPU acceleration.

Upon knowledge integration finish, \textit{Co$^2$In} completes the learning of task $t$. The entire learning process of \textit{Co$^2$In} is summarized in Algorithm~\ref{alg:training}.

\newcommand\mycommfont[1]{\footnotesize\ttfamily\textcolor{blue}{#1}}
\SetCommentSty{mycommfont}
\begin{algorithm}[t]
	\caption{Incremental learning of \textit{Co$^2$In}}
	\label{alg:training}
	\KwIn{Datasets $\{\mathcal{D}_1,\dots,\mathcal{D}_T\}$}
	\KwOut{Knowledge representations $\mathcal{A}^*$; Pattern basis $\mathcal{B}$}
	Initialize $\mathcal{A}^* = \{{\bm{A}^*}^l=\emptyset\}_{l=1}^L$, $\mathcal{B} = \{\bm{B}^l=\emptyset\}_{l=1}^L$;\\
	\For{$t=1,\dots,T$}{
        Initialize $\mathcal{A}_t = \{\bm{A}_t^l=\emptyset\}_{l=1}^L$\\
        // knowledge encoding\\
        $\mathcal{W}_t \longleftarrow \mathop{\arg\min}\limits_{\mathcal{W}} L_{ce}(\mathcal{W}, \mathcal{D}_t) + \lambda L_{reg}(\mathcal{W})$\\ 
        // knowledge consolidation\\
        \For{$l=1,\dots,L$}{
            $\bm{U}^l_{t}\bm{\Sigma}^l_{t} {\bm{V}^l_{t}}^\top \longleftarrow \text{SVD}\left(\bm{Z}^l_t - {\bm{B}^l}{\bm{B}^l}^\top \bm{Z}^l_t\right)$\\
            $\bm{B}^l \longleftarrow [\bm{B}^l, {\bm{U}^l_{t}}]$\\
            $\bm{A}_t^l \longleftarrow {\bm{B}^l}^\top \bm{W}_t^l$ \\
        }
        $\mathcal{A}_t \longleftarrow \mathop{\arg\min}\limits_{\mathcal{A}_t} L_{ce}(\mathcal{B}\mathcal{A}_t, \mathcal{D}_t)$ \\
        // knowledge integration\\
        \For{$l=1,\dots,L$}{
            ${\bm{A}^*}^l \longleftarrow [{\bm{A}^*}^l, \bm{A}_t^l]$ \\
        }
	}
    return $\mathcal{A}^*$, $\mathcal{B}$
\end{algorithm}

\begin{figure*}[!t]
    \centering
    \includegraphics[width=0.9\textwidth]{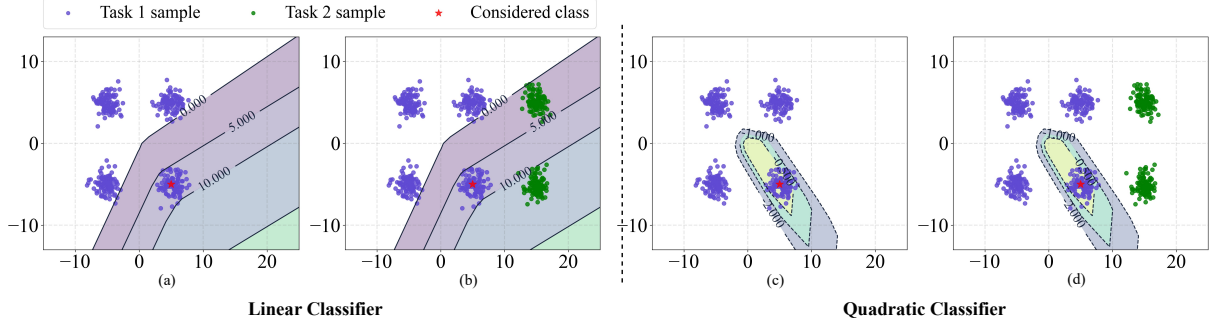}
    \caption{
        Contour map for logits corresponding to a task-1 class (cluster with the red star) output by a piece-wise linear or quadratic classifier trained with task-1 samples. 
        The linear classifier accurately identifies task-1 samples (a). Due to linear extrapolation, it produces relatively high logits for task-2 samples (b), which could result in misclassification.
        In contrast, the logits output by the quadratic classifier can not only contribute reasonable classification within task-1 samples (c) but is also suitable for task-2 samples (d).
    }
    \label{fig:toy_dataset}
\end{figure*}

\subsection{Extension to Convolutional Neural Networks} \label{sec:cnn}

In previous sections, the discussion revolved around the application of \textit{Co$^2$In} within fully connected neural networks. This section presents a straightforward strategy for extending our method to convolutional neural networks. The fundamental principle behind this extension involves converting network layers between convolution form and matrix-multiplication form.

\textbf{Knowledge consolidation with matrix-multiplication form.} 
Convolution is constituted by pixel-wise multiplications between the convolutional kernel and the image patches. This inherent linearity inspires us to reshape both the kernel and the patches into matrix form. As illustrated in Fig.~\ref{fig:cnntofc}, the $l$-th convolution layer with a kernel size of $c^{l+1} \times c^{l} \times k^l \times k^l$ is transformed into a matrix of dimensions $c^{l+1} \times(c^{l} \times k^l \times k^l)$. This transformation enables the seamless application of knowledge consolidation methods to convolutional layers without any modification.

\textbf{Knowledge integration with convolution form.} After knowledge consolidation, the convolution kernel of $l$-th layer is factorized into two matrices: a pattern basis matrix with the size of $c^{l+1} \times r_t^{l}$ and a knowledge representation matrix with the size of $r_t^{l} \times (c^l \times k^l \times k^l)$. An intriguing aspect of this decomposition is that the pattern basis can be reinterpreted as a $1\times 1$ convolution kernel, while the knowledge representation can be converted into a convolution kernel with dimensions $r_t^{l} \times c^{l} \times k^l \times k^l$, as visualized in Fig.~\ref{fig:process} (c). In line with the knowledge integration process on fully connected layers, all kernel-formed knowledge representations are encapsulated within a grouped convolutional layer to facilitate efficient inference, as shown in Fig.~\ref{fig:knowledge_integration}.

\subsection{Quadratic Classifier} \label{sec:quadratic_classifier}

\textit{Co$^2$In} remains an inherent challenge in Class-IL: the lack of cross-task discriminability during inference because of the unavailability of task identification. Recall that Eq.~\eqref{equation:classification} derives the classification result by comparing the prediction logits obtained from all $t$ tasks, each functions as an autonomous subnetwork. However, as each subnetwork undergoes an independent training process, their output logits may not cross-task discriminative, potentially influencing the Class-IL performance of \textit{Co$^2$In}.

We conducted an exploratory experiment through a toy IL dataset to assess the cross-task discriminability of outputs from these independently trained subnetworks. 
The toy dataset consists of six classes, each sampled from a Gaussian distribution. As depicted in Fig.~\ref{fig:toy_dataset}, the first four classes (colored purple) belong to task 1, while the remaining two classes (colored green) come from task 2. A two-layer fully connected model was trained exclusively on task-1 data. Contour maps were subsequently generated to analyze the behavior of logits across both tasks. Fig.~\ref{fig:toy_dataset} (a) demonstrates that the logits enable accurate classification within task 1, evidenced by higher values for corresponding class samples and lower values for other samples within the same task. Conversely, as depicted in Fig.~\ref{fig:toy_dataset} (b), the logits for task-2 samples become disproportionately large, leading to misclassification as task-1 classes. This behavior highlights the infeasibility of directly concatenating these misleading logits described in Eq.~\eqref{equation:classification}.

One potential explanation for the above issue can be attributed to the ``push'' factor~\cite{yang2022we} inherent in the cross-entropy loss. To be more specific, the cross-entropy loss can be decomposed into two components:
\begin{equation}
  \label{equation:linear_celoss} 
  \begin{aligned}
    &L_{ce}(\mathcal{W}, \mathcal{D}_t)= \!\!\!\!\!\! \sum_{(\bm{x}_k, y_k) \in \mathcal{D}_t} \!\!\!\!\!\!\! -\log\frac{\exp({\bm{W}^{L}_{y_k}}^\top \bm{x}_k^{L-1})}{\sum_i \exp({\bm{W}^{L}_i}^\top \bm{x}_k^{L-1})} \\
    &=\!\!\!\sum_{(\bm{x}_k, y_k) \in \mathcal{D}_t} \!\!
    \underbrace{-{\bm{W}^{L}_{y_k}}^\top \bm{x}_k^{L-1}}_{\text{pull term}}
    \!+ \!
    \underbrace{\log \sum_i \exp({\bm{W}^{L}_i}^\top \bm{x}_k^{L-1})}_{\text{push term}}
    ,
  \end{aligned}
\end{equation}
where $\bm{W}^{L}_i$ denotes the parameter of the final layer corresponding to class $i$, and $\bm{x}_k^{L-1}$ represents the feature of $\bm{x}_k$ extracted by the penultimate layer. 
In the last line of Eq.~\eqref{equation:linear_celoss}, the first term behaves as a ``pull'' factor that attracts the feature towards the classifier prototype $\bm{W}^{L}_{y_k}$, while the second term serves as the ``push'' factor that repels the features from the prototypes of other classes  $\bm{W}^{L}_i$.
However, the ``push'' factor is unbounded and tends to dominate the overall loss in data imbalance scenarios \cite{yang2023neural}, such as incremental learning. Figure~\ref{fig:toy_dataset} (b) evidences how the pronounced effect of the ``push'' factor can incur misclassifications.

To address this issue, inspired by Deep-SVDD \cite{deepsvdd}, we introduce a quadratic classifier for \textit{Co$^2$In} designed to emphasize the ``pull'' factor in the cross-entropy loss and enhance recognition of out-of-distribution (i.e., other tasks) classes.
The quadratic classifier is defined as follows:
\begin{equation*}
  \begin{aligned}
    \bm{Z}^L = ({\bm{W}^{L}}^\top \bm{X}^{L-1} - \bm{b})^2,
  \end{aligned}
\end{equation*}
where $\bm{b}$ is a learnable centroid, $(.)^2$ is element-wise square operation.
Combining the quadratic classifier with the cross-entropy loss yields:
\begin{equation}
  \label{equation:quadratic_celoss} 
  \begin{aligned}
    L_{ce}(\mathcal{W}, \mathcal{D}_t)&=\!\!\!\!\!\!\!\!\sum_{(\bm{x}_k, y_k) \in \mathcal{D}_t} \!\!\!\!\!\!\!\! \!- \! \log \!\frac{\exp(-({\bm{W}^{L}_{y_k}}^\top \bm{x}_k^{L-1} - \bm{b}_{y_k})^2)}{\sum_i \!\exp(-({\bm{W}^{L}_i}^\top \bm{x}_k^{L-1} - \bm{b}_i)^2)} \\
    =\sum_{(\bm{x}_k, y_k) \in \mathcal{D}_t}& 
    \underbrace{({W^{L}_{y_k}}^\top \bm{x}_k^{L-1} - \bm{b}_{y_k})^2}_{\text{pull term}} \\
    & + \underbrace{\log \sum_i \exp(-({\bm{W}^{L}_i}^\top \bm{x}_k^{L-1} - \bm{b}_i)^2)}_{\text{push term}}
    ,
  \end{aligned}
\end{equation}
where $\bm{b}_i$ denotes the centroid parameter corresponding to class $i$.
In Eq.~\eqref{equation:quadratic_celoss}, the ``pull'' term aims to quadratically draw features into the vicinity of the class centroid $\bm{b}_{y_k}$. Conversely, the ``push'' term diminishes to zero as features diverge from other class centroid $\bm{b}_i$. 
Therefore, by bounding the ``push'' factor, the quadratic classifier minimizes its impact on the overall loss. This aspect is empirically substantiated in Fig.~\ref{fig:toy_dataset} (c) and (d), where the quadratic classifier establishes a much tighter decision boundary compared to the linear classifier, ensuring cross-task discriminability.

The quadratic classifier is not exclusive to Co$^2$In. We identify two key factors that influence its effectiveness. First, the stability of the method: the compact decision boundaries of the quadratic classifier make it highly sensitive to variations in the extracted features, as the logits change quadratically in response to these variations. Second, the model’s ability to learn new features, such as the ability to expand the basis in Co$^2$In. The new features serve as a clear signal for distinguishing between new and old task classes. Methods that possess both of these characteristics, such as GPM, can benefit from combining with the quadratic classifier to strengthen their Class-IL capability, whereas other methods may not achieve the same improvement. For more details, please refer to Section~\ref{sec:exp_quadratic_classifier}.




\section{Experiment} \label{sec:experiment}

\subsection{Settings}

\textbf{Datasets.} 
We choose three datasets during the evaluation: CIFAR-10 \cite{cifar10}, CIFAR-100 \cite{cifar10}, and Tiny-ImageNet \cite{tinyimagenet}. CIFAR-10 includes ten classes of $32 \times 32$ color images. We divide it into five incremental batches, with each batch containing samples of 2 classes. CIFAR-100 consists of 60000 $32 \times 32$ color images that fall into 100 classes. We split it into ten incremental batches in most of our experiments. Tiny-ImageNet comprises 200 categories of $64 \times 64$ color images. The 200 classes are divided into ten tasks equally, each including samples from 20 classes.

\textbf{Comparison baselines.} 
We select several classic and latest IL methods for comparison, including Gradient Episodic Memory (GEM) \cite{gem}, Dark Experience Replay (DER and DER++) \cite{der}, Contrastive Continual Learning (Co2L) \cite{co2l},  Hierarchical Correlations Replay \cite{hcr}, NPCL \cite{npcl}, Learning without Forgetting (LwF) \cite{lwf}, Synaptic Intelligence (SI) \cite{si}, Orthogonal Weights Modification (OWM) \cite{OWM}, DeepInversion (DI) \cite{di}, ILCOC \cite{ilcoc}, Gradient Projection Memory (GPM) \cite{gpm}, Adam-NSCL \cite{adam-nscl}, Prototype Augmentation and Self-Supervision (PASS) \cite{pass}, Efficient Feature Transformations (EFT) \cite{eft}, Always Be Dreaming (ABD) \cite{abd}, Filter Atom Swapping (FAS) \cite{fas}, Winning SubNetworks (WSN) \cite{wsn}, DCPOC \cite{dcpoc}, VDFD \cite{vdfd}. We also introduce a fundamental method named Joint, which is trained on all task data jointly. Joint does not adhere to the IL training setup, which is generally seen as the upper-bound performance. In addition, we report the results of the conference version of our method, referred to as BMKP \cite{bmkp}. BMKP can be thought of as a simplified \textit{Co$^2$In} without the knowledge integration and the quadratic classifier.

\textbf{Performance metrics.} 
We follow \cite{gem, der} and employ classification accuracy (ACC) on all learned tasks after IL as the primary metric for performance evaluation. To eliminate the stochasticity inherent in neural network training, we report the average performance of five runs with random seeds. We also provide memory usage in megabytes (MB) for certain methods that require additional storage beyond the network parameters. The additional storage includes exemplars for DER, gradient basis for GPM and Adam-NSCL, and long-term memory for BMKP and \textit{Co$^2$In}.

\begin{table*}
    \caption{
      Average accuracy (\%) and the standard deviation over five runs with random seeds. Joint is the model jointly trained with data from all tasks. (-) means that the result was unavailable, due to the intractable training time by our implementation (GEM) or unsuitable setting (GPM and Adam-NSCL for Class-IL). ($\dag$) and ($\ddag$) imply the results are quoted from HCR and the origin paper, where the standard deviations is not provided. \textbf{Bold} and \underline{underlined} denote the best and the second-best performance. Results with (*) imply the leveraging of rotate class augmentation.
    }
    \label{tab:compare}
    \centering
  \renewcommand{\arraystretch}{0.6} 
    \setlength{\tabcolsep}{3.0mm}{
    \begin{tabular}{l|c|c|c|c|c|c|c|c}
      \toprule
      \multirow{2}{*}{\textbf{Methods}}   & \multirow{2}{*}{\textbf{Venue}} & \multirow{2}{*}{\textbf{Buffer}}    & \multicolumn{2}{c|}{\textbf{CIFAR-10}} & \multicolumn{2}{c|}{\textbf{CIFAR-100}} & \multicolumn{2}{c}{\textbf{Tiny-ImageNet}}    \\
      
       & & & Class-IL & Task-IL & Class-IL & Task-IL & Class-IL & Task-IL \\
      
      \midrule              
      Joint                  & -            & -   & 91.29         & 98.07         & 70.31         & 91.18           & 58.07         & 82.01             \\
      \midrule                                 
      GEM \cite{gem}         & NIPS2017     & 500 & 23.66$\pm$1.2 & 85.14$\pm$2.1 & 11.76$\pm$0.7 & 62.80$\pm$2.7   & -             & -              \\
      DER \cite{der}         & NIPS2020     & 500 & 70.57$\pm$1.0 & 93.41$\pm$0.3 & 34.24$\pm$1.4 & 73.26$\pm$1.3   & 17.75$\pm$1.1 & 51.78$\pm$0.9 \\
      DER++ \cite{der}       & NIPS2020     & 500 & 73.48$\pm$0.3 & 94.10$\pm$0.3 & 36.52$\pm$2.0 & 74.86$\pm$1.1   & 19.38$\pm$1.4 & 51.91$\pm$0.7 \\
      Co2L$^\dag$ \cite{co2l}& ICCV2021     & 500 & \underline{74.26}& \underline{95.90}& 30.45   & 79.61           & 20.12         & 53.04         \\
      HCR$^\dag$ \cite{hcr}  & KBS2022      & 500 & \textbf{74.64}& \textbf{96.01}& 33.71         & 80.55           & 22.05         & 54.73         \\
      NPCL$^\ddag$ \cite{npcl}& NIPS2023    & 500 & 71.34$\pm$1.5 & -             & 37.43         & -               & 15.29$\pm$1.0 & -              \\
      \midrule                              
      LwF \cite{lwf}         & TPAMI2017    & -   & 19.59$\pm$0.1 & 91.91$\pm$0.7 & 10.05$\pm$0.6 & 63.78$\pm$4.3   & 9.36$\pm$0.1  & 58.61$\pm$1.8         \\
      SI \cite{si}           & ICML2017     & -   & 19.47$\pm$0.2 & 69.40$\pm$1.9 & 7.42$\pm$0.5  & 62.21$\pm$2.6   & 6.77$\pm$0.1  & 60.91$\pm$1.3  \\
      OWM  \cite{OWM}        & NMI2019      & -   & 47.50$\pm$0.5 & 81.00$\pm$0.1 & 27.63$\pm$0.5 & 60.64$\pm$0.4   & 15.30$\pm$0.3 & 40.89$\pm$0.3  \\
      DI  \cite{di}          & CVPR2020     & -   & 24.14$\pm$2.2 & 94.46$\pm$0.6 & 11.19$\pm$0.3 & 68.43$\pm$2.1   & 8.19$\pm$0.2  & 66.12$\pm$0.9  \\
      ILCOC \cite{ilcoc}     & CVPRW2021    & -   & 38.71$\pm$4.2 & 94.74$\pm$0.3 & 22.19$\pm$1.6 & 69.87$\pm$1.0   & 15.78$\pm$0.4 & 54.09$\pm$0.7  \\
      GPM \cite{gpm}         & ICLR2021     & -   & -             & 86.58$\pm$0.9 & -             & 70.93$\pm$0.9   & -             & 59.84$\pm$0.2  \\
      Adam-NSCL \cite{adam-nscl} & CVPR2021 & -   & -             & 87.23$\pm$0.4 & -             & 65.69$\pm$0.2   & -             & 59.98$\pm$0.7  \\
      PASS \cite{pass}       & CVPR2021     & -   & 53.12$\pm$0.3$^*$ & 86.07$\pm$0.2$^*$ & 31.80$\pm$0.7$^*$ & 77.30$\pm$0.4$^*$   & \underline{28.48$\pm$0.6}$^*$ & 62.87$\pm$0.4$^*$   \\
      EFT \cite{eft}         & CVPR2021     & -   & 55.12$\pm$1.9 & 95.09$\pm$0.3 & 35.99$\pm$1.0 & 79.29$\pm$0.6   & 26.89$\pm$0.6 & 63.88$\pm$0.5  \\
      ABD \cite{abd}         & ICCV2021     & -   & 57.64$\pm$0.3 & 95.11$\pm$0.3 & 33.30$\pm$0.3 & 74.83$\pm$0.5   & 15.80$\pm$0.4 & 46.76$\pm$0.6   \\      
      FAS \cite{fas}         & ICLR2022     & -   & 43.36$\pm$2.8 & 90.89$\pm$1.3 & 25.79$\pm$0.7 & 70.89$\pm$0.6   & 24.29$\pm$0.3 & 60.10$\pm$0.2   \\
      WSN \cite{wsn}         & ICML2022     & -   & 19.14$\pm$0.1 & 92.99$\pm$0.4 & 8.79$\pm$0.2  & 81.10$\pm$0.7   & 7.68$\pm$0.3  & 67.39$\pm$0.7   \\
      DCPOC \cite{dcpoc}     & PR2023       & -   & 44.30$\pm$0.6 & 90.43$\pm$0.3 & 25.80$\pm$0.4 & 72.20$\pm$0.2   & 19.75$\pm$0.1 & 53.08$\pm$0.2   \\
      VDFD \cite{vdfd}       & TPAMI2023    & -   & 59.90$\pm$1.8$^*$ & 94.24$\pm$0.5$^*$ & \textbf{38.38$\pm$0.9}$^*$& \underline{83.30$\pm$0.2}$^*$& 26.21$\pm$0.7$^*$ & 66.62$\pm$0.3$^*$       \\
      \midrule                                              
      BMKP (ours)            & CVPR2023     & -   & 29.70$\pm$1.0 & 94.49$\pm$0.2 & 16.89$\pm$0.6 & 79.62$\pm$0.8   & 8.91$\pm$0.2  & \textbf{70.36$\pm$0.2}      \\
      \textit{Co$^2$In} (ours) & - & - & 56.21$\pm$5.0$^*$ & 95.21$\pm$0.3$^*$ & \underline{37.54$\pm$0.9} & \textbf{84.21$\pm$1.0} & \textbf{31.23$\pm$0.5} & \underline{68.75$\pm$0.7}          \\
      \bottomrule
    \end{tabular}
    }
\end{table*}

\begin{figure}[!t]
    \centering
    \includegraphics[width=0.9\columnwidth]{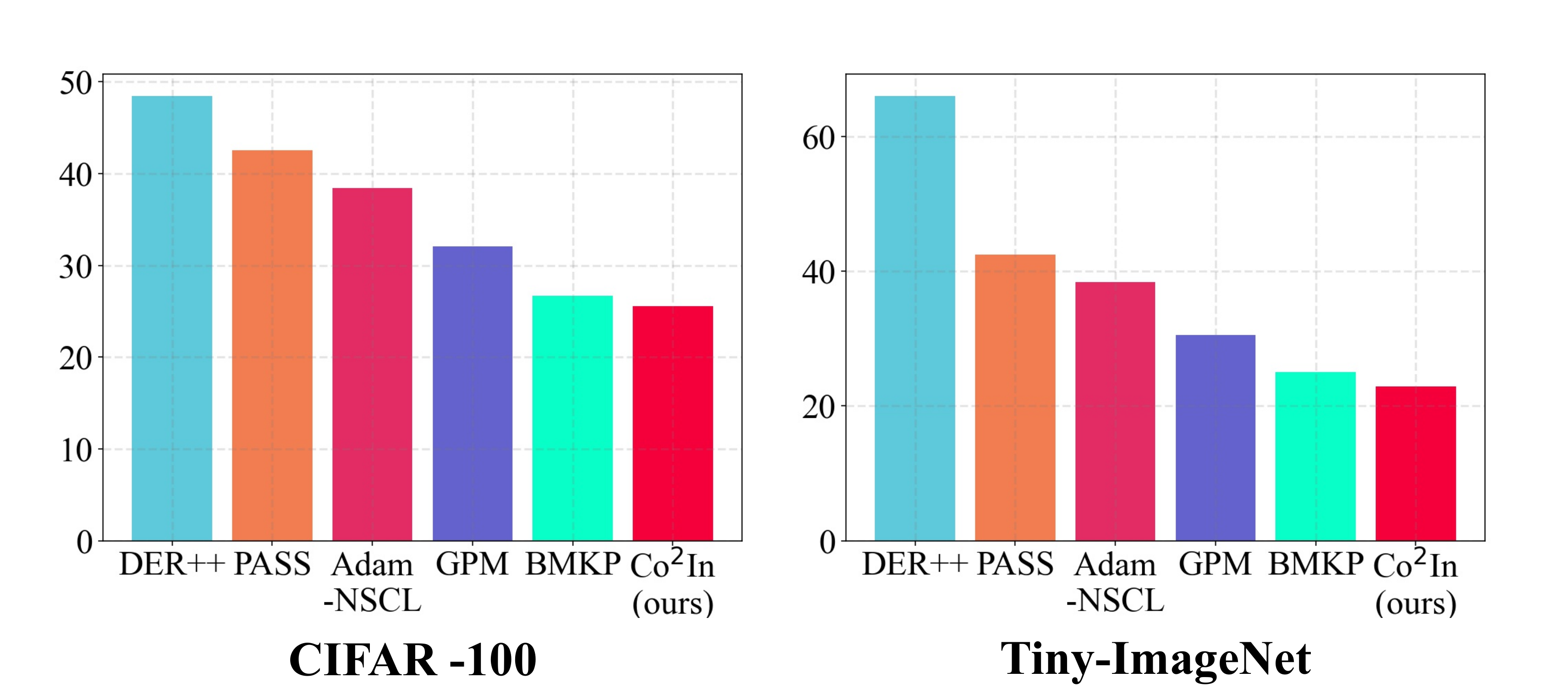}
    \caption{
        Memory usage (MB) on CIFAR-100 and Tiny-ImageNet.
    }
    \label{fig:memory}
\end{figure}

\textbf{Implementation details.}
To ensure a fair performance comparison, we evaluated the performance of all methods under similar memory constraints. For the projection-based methods (GPM, Adam-NSCL, BMKP, and \textit{Co$^2$In}) in which the memory usage depends on the width of the network, we follow \cite{gpm} and employ a narrower version of ResNet18 \cite{resnet}. In contrast, all the other baselines employ the standard ResNet18 as the network backbone. For rehearsal-based methods (GEM, DER, DER++, Co2L, HCR) that necessitate additional memory to preserve exemplars, we equip them with a buffer capable of storing up to 500 samples.

\textbf{Hyper-parameters.}
All hyper-parameters are searched through a validation set obtained by sampling ten percent samples from the training set.
We optimize \textit{Co$^2$In} with stochastic gradient descent (SGD) optimizer over 100 iterative epochs.
We tailor the batch size to 32, 16, and 128 for CIFAR-10, CIFAR-100, and Tiny-ImageNet. The learning rates are configured at 2e-3, 2e-2, and 2e-2 for the three datasets, respectively.
The balance weight $\lambda$ is tuned within the range of $\{0, 0.1, 1, 10, 100\}$, with the optimal values selected as 0.1, 1, and 10 for CIFAR-10, CIFAR-100, and Tiny-ImageNet. Additionally, we adjust the retained singular-value ratio (Cf. Eq.~\eqref{equation:new_pattern}) across the range of $\{95\%, 96\%, 97\%, 98\%, 99\%\}$, ultimately settling on values of 97\%, 97\%, and 98\% for the three datasets, respectively. For more details, please refer to the code repository \footnote{https://github.com/SunWenJu123/Co2In}.

\begin{table*}[h]
  \caption{
    Ablation study of \textit{Co$^2$In} on CIFAR-100 and Tiny-ImageNet.
  }
  \label{tab:ablation}
  \centering
  \renewcommand{\arraystretch}{0.6} 
  \setlength{\tabcolsep}{15pt}  
  \begin{tabular}{l|cc|cc}
    \toprule
    \multirow{2}{*}{\textbf{Methods}} & \multicolumn{2}{c|}{\textbf{CIFAR-100}} & \multicolumn{2}{c}{\textbf{Tiny-ImageNet}}          \\
    & Class-IL & Task-IL & Class-IL & Task-IL \\
    \midrule
    (a) \textit{Co$^2$In} w/o representation compaction regularizer  & 34.75$\pm$0.5 & 83.04$\pm$0.5 & 27.05$\pm$2.0 & 65.00$\pm$2.2   \\
    (b) \textit{Co$^2$In} w/o basis updating                         & 6.88$\pm$5.9  & 36.48$\pm$17.7& 2.90$\pm$1.0  & 27.23$\pm$3.6   \\
    (c) \textit{Co$^2$In} w/o retraining                             & 36.12$\pm$0.7 & 81.78$\pm$3.0 & 28.91$\pm$0.7 & 65.68$\pm$1.9   \\
    (d) \textit{Co$^2$In} w/o quadratic classifier                   & 16.90$\pm$0.5 & 79.65$\pm$0.4 & 8.81$\pm$0.1  & 69.31$\pm$0.5   \\
    \textit{Co$^2$In}                                                & 37.54$\pm$0.9 & 84.21$\pm$1.0 & 31.23$\pm$0.5 & 68.75$\pm$0.7   \\
    \bottomrule
  \end{tabular}
\end{table*}
 
\subsection{Comparison Result} 

This section provides a comparison between \textit{Co$^2$In} and baseline methods. The performance is presented in Table~\ref{tab:compare} for both Task-IL and Class-IL settings. Additionally, Fig.~\ref{fig:memory} illustrates the memory usage of various methods.

For the small-scale CIFAR-10 dataset, \textit{Co$^2$In} is competitive with the rotate class augmentation technique\footnote{Augment class by rotating the training data in 90, 180, and 270 degrees, a widespread incremental learning technique adopted by the comparison methods PASS\cite{pass} and VDFD\cite{vdfd}.}, showcasing the best Task-IL performance and the third-best Class-IL performance among exemplar-free methods. Our empirical findings suggest that the rotate class augmentation is particularly efficacious in enhancing performance in settings with limited numbers of categories. Section~\ref{sec:ssl_exp} details an in-depth examination of this augmentation.

For CIFAR-100, \textit{Co$^2$In} achieves the second-best Class-IL performance (37.54\%) and the best Task-IL performance (84.21\%), surpassing the second-best VDFD by 0.91\%. For Tiny-ImageNet, under the Task-IL setting, the conference version BMKP obtains the best classification accuracy (70.36\%), while \textit{Co$^2$In} delivers the second-best performance (68.75\%). The difference in performance can be attributed to the distinct classifier employed: \textit{Co$^2$In} relies on quadratic classifiers, while BMKP utilizes linear classifiers. Additionally, under the Class-IL setting, \textit{Co$^2$In} attains the best performance (31.23\%) on Tiny-ImageNet, surpassing the second-best method by 2.75\%.


Furthermore, the competitive performance of \textit{Co$^2$In} is achieved with lower memory utilization. As depicted in Fig.~\ref{fig:memory}, \textit{Co$^2$In} utilizes 25.53 MB of memory when applied to CIFAR-100, which is lower than that of the second-best method by a margin of 1.18 MB. Similarly, on Tiny-ImageNet, \textit{Co$^2$In} consumes 22.89 MB of memory, outperforming the second-best BMKP with 2.16 MB. These results underscore the effectiveness of \textit{Co$^2$In} on its economical use of resources without compromising performance across diverse datasets and learning paradigms.

In conclusion, the aforementioned results substantiate that \textit{Co$^2$In} strikes a good balance between plasticity and stability in incremental learning, attributed to its brain-inspired bi-level memory framework. Additionally, by compacting the decision boundaries, the quadratic classifier effectively boosts the Class-IL performance of \textit{Co$^2$In}.

\begin{figure*}[!t]
    \centering
    \includegraphics[width=0.99\textwidth]{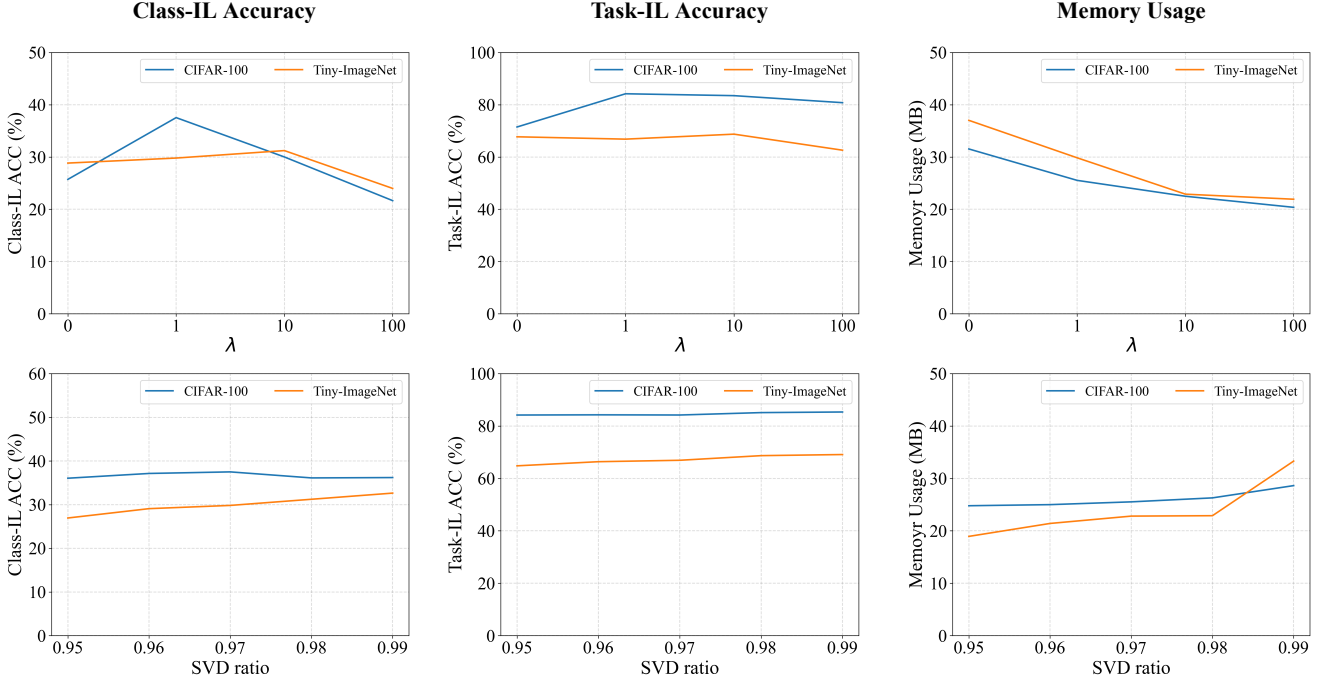}
    \caption{
        Incremental learning results of \textit{Co$^2$In} with different $\lambda$ (first row) and SVD ratio (second row) on CIFAR-100 and Tiny-ImageNet.
    }
    \label{fig:sensitive}
\end{figure*}

\subsection{Ablation Study}

This section presents ablation studies on each component.

\textbf{Effectiveness of representation compaction regularizer.}
We start by introducing variation (a), which is derived from \textit{Co$^2$In} by setting the hyperparameter $\lambda$ to 0. As demonstrated in the first row of Table~\ref{tab:ablation}, variation (a) exhibits a degraded performance relative to the original \textit{Co$^2$In}. This observation reveals that the representation compaction regularizer is able to enhance the performance of \textit{Co$^2$In}.

\textbf{Effectiveness of basis updating.} 
In variation (b), knowledge consolidation does not incorporate the basis updating step. Consequently, the knowledge representations of variation (b) are acquired with respect to the task-1 pattern basis. The second row of Table~\ref{tab:ablation} shows a decline in performance for variation (b) across all setups. 
This phenomenon can be attributed to the fact that basis updating expands the range of CKS. Without this expansion, the expression of knowledge would be severely constrained, i.e., low plasticity.

\textbf{Effectiveness of retraining.} 
Based on \textit{Co$^2$In}, variation (c) excludes the retraining step (Cf. Eq.~\eqref{equation:rehearsal}) after knowledge projection. The third row of Table~\ref{tab:ablation} shows that variation (c) exhibits slightly lower performance than \textit{Co$^2$In}. This phenomenon is expected since retraining serves as a fine-tuning process that refines the model performance. 


\textbf{Effectiveness of quadratic classifier.}
In the final variation (d), the inherent quadratic classifier of \textit{Co$^2$In} is substituted with a conventional linear classifier. As can be seen, the Class-IL performance of variation (d) experiences a significant drop, but the Task-IL performance remains consistent. This finding suggests that variation (d) may have a bias toward certain tasks, highlighting the importance of the quadratic classifier for the Class-IL ability of \textit{Co$^2$In}.

\subsection{Analysis of Hyper-Parameters}

\textbf{Sensitivity analysis of $\lambda$ (Cf. Eq.~\eqref{equation:loss}).} 
We first conduct experiments to evaluate the sensitivity of $\lambda$ by setting different values ($0,1,10,100,1000$), and further investigate the effectiveness of the representation compaction regularizer.
The first row in Fig.~\ref{fig:sensitive} depicts the experimental results. when $\lambda=0$ (i.e., the representation compaction regularizer is not applied), \textit{Co$^2$In} consumes the most memory but yields relatively lower accuracy on both Task-IL and Class-IL. 
As the value of $\lambda$ rises, the classification accuracy correspondingly increases, peaking at an optimal point before gradually declining. This phenomenon reveals that the representation compaction regularizer improves the performance of \textit{Co$^2$In}, which we attribute to facilitating the transfer of learned knowledge. 
Furthermore, the memory usage of \textit{Co$^2$In} consistently diminishes as $\lambda$ grows, signifying that the regularizer efficiently improves the memory efficiency of \textit{Co$^2$In} by encouraging the model to represent new knowledge using previous-established patterns.

\textbf{Sensitivity analysis of SVD ratio (Cf. Eq.~\eqref{equation:new_pattern}).}
Subsequently, we evaluate \textit{Co$^2$In} with different SVD ratios: 0.95, 0.96, 0.97, 0.98, 0.99. 
The experimental results are depicted in the second row of Fig.~\ref{fig:sensitive}. As the SVD ratio increases, both accuracy and memory usage remain stable and gradually increase monotonically. 
This phenomenon is owing to the fact that a larger SVD ratio allows \textit{Co$^2$In} to preserve more basis during knowledge consolidation, thereby retaining more knowledge.
We empirically determine an SVD ratio of $0.97$ for CIFAR100 and $0.98$ for Tiny-ImageNet to balance the performance and memory usage.

\begin{table}
  \caption{
    Incremental learning result (ACC \%) of \textit{Co$^2$In} with and without self-supervised technique on CIFAR-10. ``RCA'' indicates the rotate class augmentation technique.
  }
  \label{tab:ssl}
  \centering
  \renewcommand{\arraystretch}{0.6} 
  \setlength{\tabcolsep}{2mm}{
  \begin{tabular}{c|l|c|c}
    \toprule
    Dataset  & \textbf{Methods}                & \textbf{Class-IL}     & \textbf{Task-IL}           \\
    \midrule
    \multirow{8}{*}{\textbf{CIFAR-10}}  & OWM  \cite{OWM}         & 47.50$\pm$0.5              & 81.00$\pm$0.1         \\
                                        & PASS \cite{pass}        & 53.12$\pm$0.3              & 86.07$\pm$0.2                  \\   
                                        & EFT \cite{eft}          & 55.12$\pm$1.9              & 95.09$\pm$0.3         \\
                                        & ABD \cite{abd}          & \textbf{57.64$\pm$0.3}     & \underline{95.11$\pm$0.3}    \\ 
                                        & FAS \cite{fas}          & 43.36$\pm$2.8              & 90.89$\pm$1.3               \\
                                        & DCPOC \cite{dcpoc}      & 44.30$\pm$0.6              & 90.43$\pm$0.3                   \\
    \cline{2-4}                    
    \rule{0pt}{9pt}                     & BMKP (ours)             & 29.70$\pm$1.0              & 94.49$\pm$0.2                  \\
                                        & \textit{Co$^2$In} (ours)         & 38.15$\pm$2.9              & 92.74$\pm$0.7                  \\
                                        & \textit{Co$^2$In} w/ RCA (ours)  & \underline{56.21$\pm$5.0}  & \textbf{95.21$\pm$0.3}       \\
    
    \bottomrule
  \end{tabular}
  }
\end{table}

\begin{table}[t]
\centering
\caption{
Incremental learning result (ACC \%) of Co$^2$In with and without rotate class augmentation (RCA) with different numbers of tasks.
}
\label{tab:rca_tasks}
\resizebox{0.85\columnwidth}{!}{
  \renewcommand{\arraystretch}{0.6} 
\begin{tabular}{l|c|cc}
\midrule                 
\textbf{Method} & \textbf{Partition}   & \multicolumn{2}{c}{\textbf{CIFAR-100}}    \\ 
                &   & Class-IL (\%) & Task-IL (\%)   \\ 
\midrule    
  \textit{Co$^2$In}                             & 10-Split   & 37.54          &  84.21 \\ 
  \rowcolor{gray!20} \textit{Co$^2$In w/ RCA}   & 10-Split   & 29.58          &  84.18 \\
  \textit{Co$^2$In}                             & 25-Split   & 21.75          &  91.45 \\ 
  \rowcolor{gray!20} \textit{Co$^2$In w/ RCA}   & 25-Split   & 27.37          &  91.03 \\ 
  \textit{Co$^2$In}                             & 50-Split   & 5.02           &  91.53 \\ 
  \rowcolor{gray!20} \textit{Co$^2$In w/ RCA}   & 50-Split   & 21.73          &  94.56 \\ 
\midrule    
\textbf{Method} & \textbf{Partition}   & \multicolumn{2}{c}{\textbf{Tiny-ImageNet}}    \\ 
                &   & Class-IL (\%) & Task-IL (\%)    \\ 
\midrule    
  \textit{Co$^2$In}                             & 10-Split   &  31.23         &  68.75\\ 
  \rowcolor{gray!20} \textit{Co$^2$In w/ RCA}   & 10-Split   &  24.07         &  63.25\\ 
  \textit{Co$^2$In}                             & 25-Split   &  20.00         &  78.78\\ 
  \rowcolor{gray!20} \textit{Co$^2$In w/ RCA}   & 25-Split   &  22.00         &  76.98\\ 
  \textit{Co$^2$In}                             & 50-Split   &  7.03          &  83.77\\ 
  \rowcolor{gray!20} \textit{Co$^2$In w/ RCA}   & 50-Split   &  14.40         &  85.38\\ 
  \textit{Co$^2$In}                             & 100-Split  &  3.74          &  92.01\\ 
  \rowcolor{gray!20} \textit{Co$^2$In w/ RCA}   & 100-Split  &  8.41          &  92.17\\ 
\midrule  
\end{tabular}
}
\end{table}

\subsection{Analysis of Rotate Class Augmentation} \label{sec:ssl_exp}

This section delves into the impact of the rotate class augmentation (RCA) on the performance of \textit{Co$^2$In}. As shown in Table~\ref{tab:ssl}, RCA significantly boosts the performance of \textit{Co$^2$In} on CIFAR-10, where each task contains only two classes. Without RCA, the model struggles to generalize due to overfitting on narrow task-specific features. For instance, distinguishing between cars and ships may cause the model to over-rely on background color (e.g., blue water for ships), which does not transfer well to unseen classes. RCA mitigates this by enforcing feature learning through self-supervised rotation tasks, leading to improved generalization and better Class-IL performance.

To further investigate the effect of RCA, we report results on CIFAR-100 and Tiny-ImageNet with different task splits in Table~\ref{tab:rca_tasks}. These results reveal that RCA is particularly effective when the number of classes per task is small. For instance, on CIFAR-100 with 50 tasks (i.e., 2 classes per task), RCA boosts Class-IL accuracy from 5.02\% to 21.73\%. Similarly, on Tiny-ImageNet with 50 tasks, Class-IL improves from 7.03\% to 14.40\%, and with 100 tasks (2 classes per task), it improves from 3.74\% to 8.41\%. This effect diminishes when tasks already contain sufficient class diversity, such as 10-split CIFAR-100 or 10-split Tiny-ImageNet. Therefore, we recommend applying RCA alongside our \textit{Co$^2$In} in scenarios where each task involves a limited number of classes.

\begin{figure}[t]
    \centering
    \includegraphics[width=1.0\columnwidth]{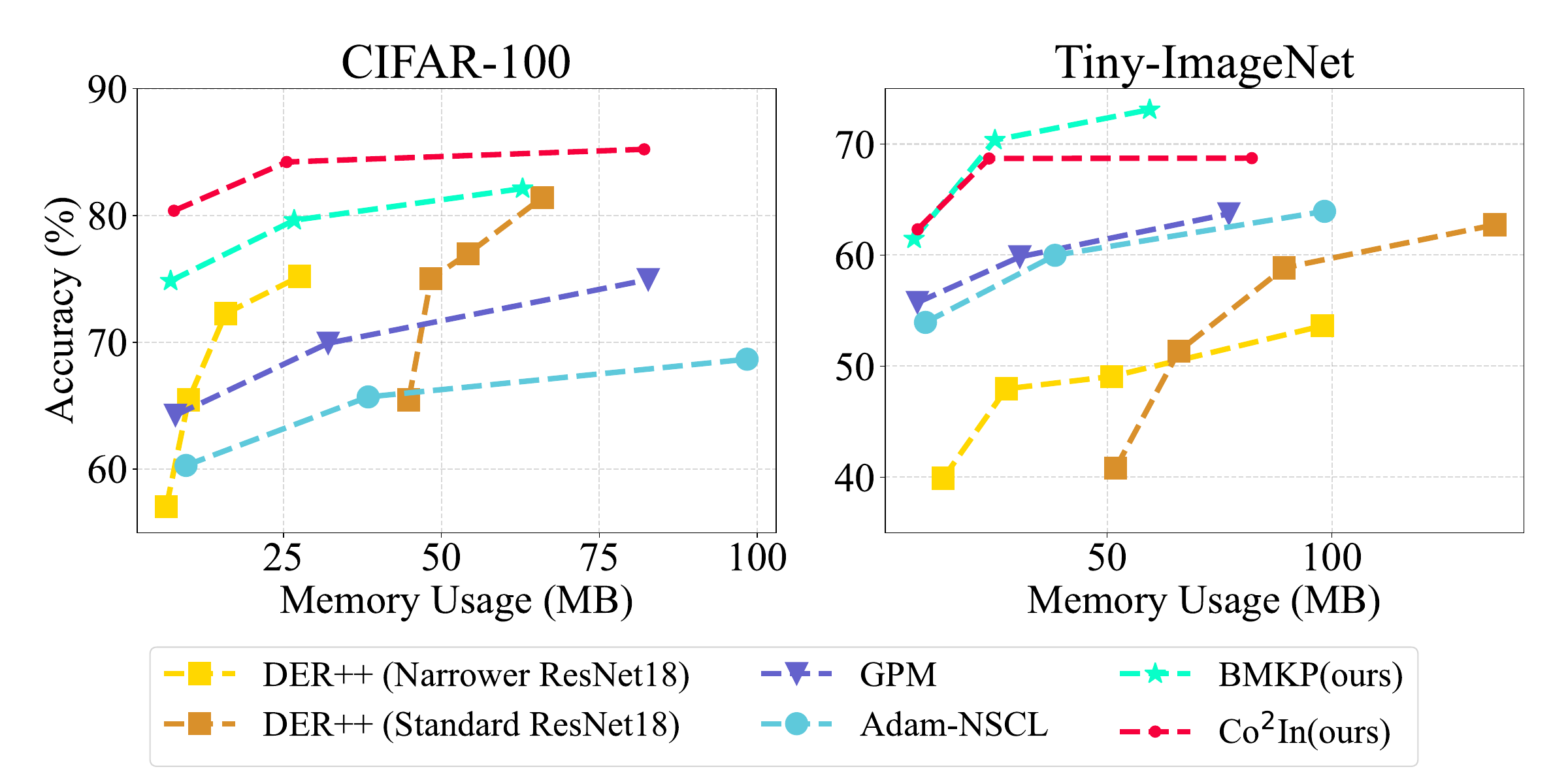}
    \caption{
        Memory usage (MB) versus Task-IL performance (\%) on CIFAR-100 and Tiny-ImageNet.
    }
    \label{fig:memory_com}
\end{figure}

\subsection{Analysis of Memory Efficiency }

This section analyzes the performance of models under different memory usage. For projection-based methods (e.g., GPM, Adam-NSCL, BMKP, and \textit{Co$^2$In}), memory usage is controlled by scaling the width of the network backbone. For the rehearsal-based method (DER++), we evaluate its performance using both the standard and narrower versions of ResNet18 across different numbers of exemplars (200, 500, 1000, 2000).
Figure~\ref{fig:memory_com} illustrates the Task-IL results on CIFAR-100 and Tiny-ImageNet. As can be seen, the performance of DER++ is sensitive to the number of exemplars. As the network grows larger, the effect of the exemplar number becomes more significant. However, as memory is limited, DER++ shows a substantial performance drop. In contrast, \textit{Co$^2$In} performs comparable memory efficiency to BMKP. Both methods achieve higher accuracy than others under the same memory budget. 
These findings demonstrate the effectiveness of \textit{Co$^2$In} in memory-sensitive scenarios.

\begin{figure*}[h]
  \centering
  \includegraphics[width=0.90\textwidth]{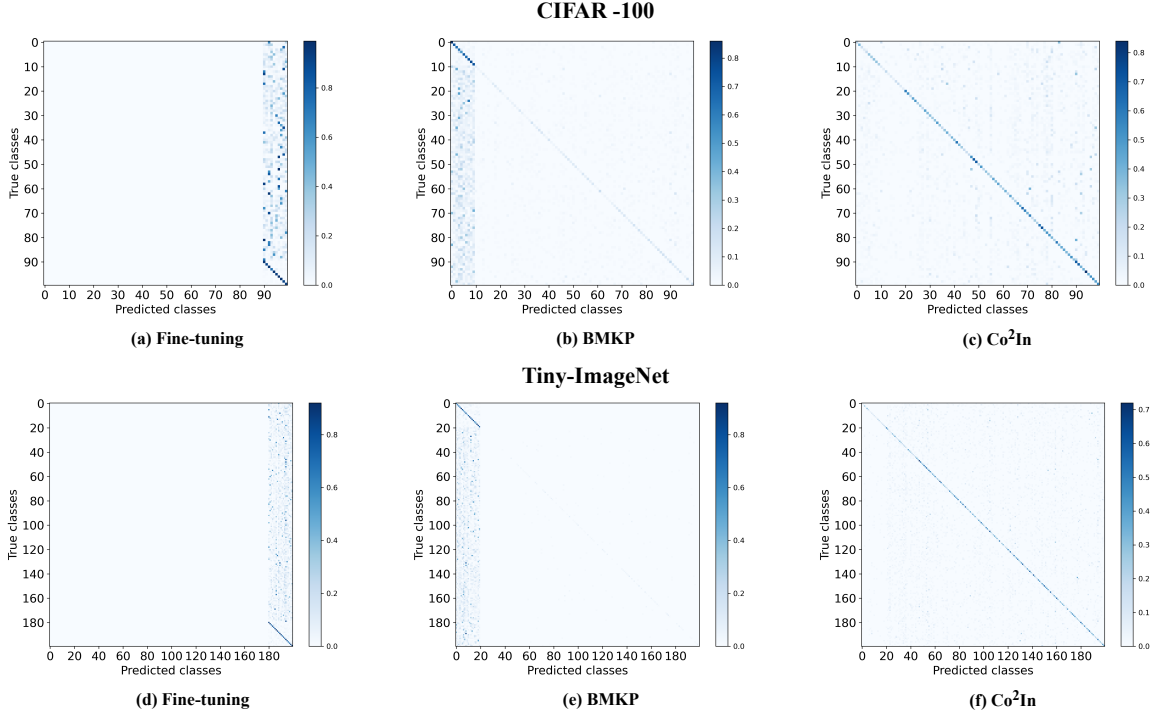}
  \caption{
    Confusion matrix of fine-tuning, BMKP, and \textit{Co$^2$In} on CIFAR-100 and Tiny-ImageNet.
  }
  \label{fig:confusion}
\end{figure*}

\subsection{Analysis of Computational Efficiency }

The knowledge integration is introduced to enhances the computational efficiency of \textit{Co$^2$In}. To validate this, we compare the FLOPs of various methods, including LWF, DER++, and \textit{Co$^2$In} with and without knowledge integration. As shown in Table~\ref{tab:flops}, knowledge integration notably reduces the computational cost in both CIFAR-100 and Tiny-ImageNet. Specifically, \textit{Co$^2$In} with knowledge integration achieves the lowest FLOPs, outperforming LWF and DER++ in both Class-IL and Task-IL settings. For example, in the CIFAR-100 Class-IL setting, \textit{Co$^2$In} reduces FLOPs by nearly 40\% compared to other methods. These results demonstrate that \textit{Co$^2$In} achieves state-of-the-art performance with significantly fewer computations. The reduction in computational cost is even more pronounced in the Task-IL setting, highlighting the computational efficiency of \textit{Co$^2$In}.

\begin{table}[t]
\centering
\caption{FLOPs for different methods on CIFAR-100 and Tiny-ImageNet.}
\label{tab:flops}
\resizebox{1.0\columnwidth}{!}{
  \renewcommand{\arraystretch}{0.6} 
\setlength{\tabcolsep}{1pt} 
\begin{tabular}{l|cc|cc}
\midrule                 
 \textbf{Method} & \multicolumn{2}{c}{\textbf{CIFAR-100}}  & \multicolumn{2}{c}{\textbf{Tiny-ImageNet}}  \\ 
 
                   & Class-IL & Task-IL & Class-IL & Task-IL   \\ 
\midrule    
  LWF \cite{lwf}                                & 5.57$\times 10^8$           &  5.57$\times 10^8$  &  2.22$\times 10^9$          &  2.22$\times 10^9$ \\ 
  DER++ \cite{der}                              & 5.57$\times 10^8$           &  5.57$\times 10^8$  &  2.22$\times 10^9$          &  2.22$\times 10^9$ \\ 
  \textit{Co$^2$In} w/o knowledge integration   & 1.45$\times 10^9$           &  1.45$\times 10^8$  &  3.23$\times 10^9$          &  3.23$\times 10^8$ \\ 
  \textit{Co$^2$In}                             & \textbf{3.49$\times 10^8$}  &  \textbf{3.49$\times 10^7$}  &  \textbf{1.43$\times 10^9$} & \textbf{1.43$\times 10^8$}               \\ 
\midrule  

\end{tabular}
}
\end{table}

\begin{figure*}[h]
  \centering
  \includegraphics[width=0.90\textwidth]{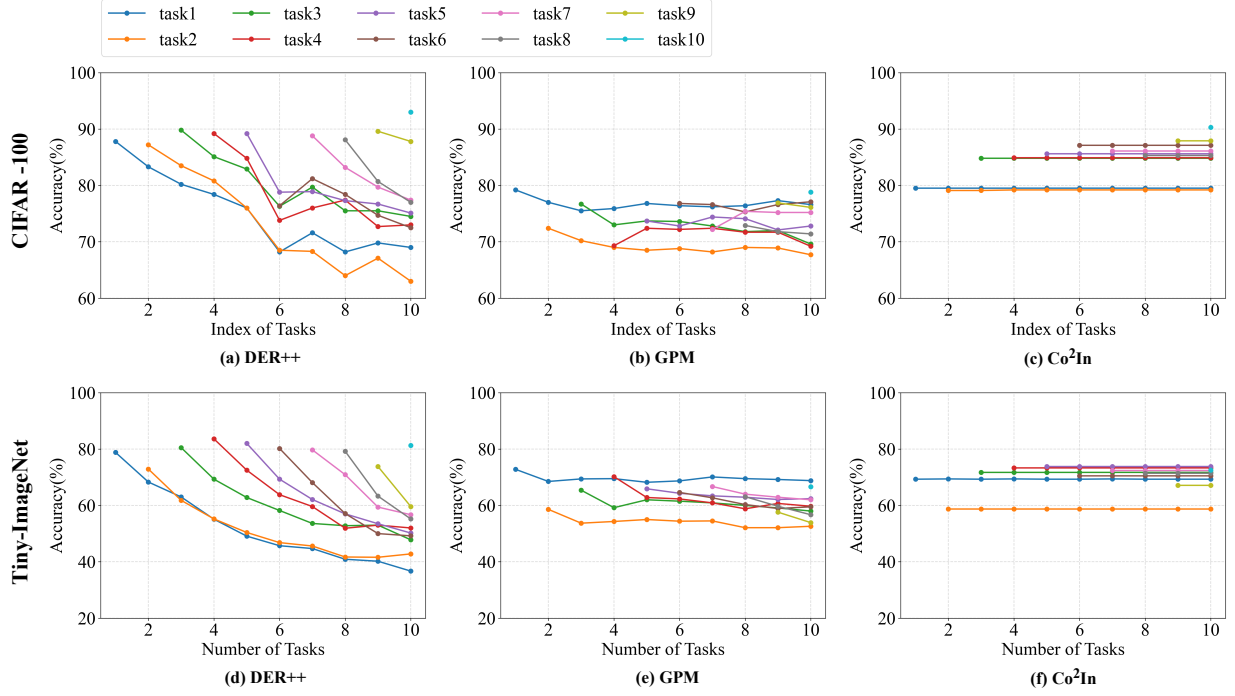}
  \caption{
    Diagram of ACC (\%) of each task during incremental learning on CIFAR-100 and Tiny-ImageNet.
  }
  \label{fig:ps}
\end{figure*}


\subsection{Analysis of Classification Balance }

We propose the quadratic classifier to boost the Class-IL capability of \textit{Co$^2$In}. To validate its effectiveness, we analyze the confusion matrices under Class-IL for fine-tuning, BMKP, and \textit{Co$^2$In} on CIFAR-100 and Tiny-ImageNet datasets. The confusion matrix reveals the model's correct predictions (represented by the diagonal entries) and misclassifications (shown by off-diagonal entries). As illustrated in Fig.~\ref{fig:confusion} (a) and (d), fine-tuning inclines to classify samples into new-task classes. Additionally, Fig.~\ref{fig:confusion} (b) and (e) show that BMKP displays a bias towards the first task, showing that linear classifiers struggle to accurately identify new-task samples, akin to the toy-dataset example provided in Fig.~\ref{fig:toy_dataset}. In contrast, Fig.~\ref{fig:confusion} (c) and (f) demonstrate that the highlights in the confusion matrices of \textit{Co$^2$In} concentrate on diagonal entries, proving that the quadratic classifier enables \textit{Co$^2$In} with task-wise balance classification ability.

\subsection{Analysis of Plasticity and Stability}

This paper introduces a bi-level memory framework for \textit{Co$^2$In} that decouples the roles of knowledge acquiring and remembering into distinct memory units, promoting both plasticity and stability. In this section, we analyze the plasticity and stability by inspecting the performance trend of each task during incremental learning. The performance curves of DER++\cite{der}, GPM\cite{gpm}, and \textit{Co$^2$In} are depicted in Fig.~\ref{fig:ps}.
As evidenced in the figure, DER++ exhibits strong plasticity as it achieves high accuracy immediately once each task is finished. However, as subsequent tasks are introduced, performance sharply declines, since the stored partial old-task exemplars may not support the original data distribution and fail to prevent forgetting, highlighting the low stability of DER++. On the contrary, GPM maintains more stable performance across all tasks owing to its gradient projection mechanism. However, the peak performance for each task is lower compared to \textit{Co$^2$In} and DER++, revealing that the gradient direction restriction leads to low plasticity of GPM. As for \textit{Co$^2$In}, the bi-level memory framework guarantees high plasticity by the working memory and non-degrading performance for each task by the long-term memory.

\begin{table}[t]
  \caption{
    Class Incremental learning results (ACC \%) on CIFAR100 for various numbers of tasks (5, 10, 20). Results with ``$\dag$'' stem from ABD \cite{abd}.
  }
  \label{tab:tasks_classil}
  \centering
  \renewcommand{\arraystretch}{0.6} 
  \begin{tabular}{l|m{1.5cm}<{\centering}|m{1.5cm}<{\centering}|m{1.5cm}<{\centering}}
    \toprule
    \textbf{Methods}     & \textbf{5-split CIFAR-100}   & \textbf{10-split CIFAR-100}    & \textbf{20-split CIFAR-100}        \\ 
    \midrule
    DGR \cite{dgr}$^\dag$      & 14.4          & 8.1                   & 4.1                 \\
    LwF \cite{lwf}$^\dag$      & 17.0          & 9.2                   & 4.7                 \\ 
    DI \cite{di}$^\dag$        & 18.8          & 10.9                  & 5.7                 \\ 
    ABD \cite{abd}$^\dag$      & 43.9          & 33.7                  & 20.0    \\ 
    PASS \cite{pass}           & \underline{45.2}& 30.8                & 17.4                \\ 
    CRHD \cite{crhd}           & \textbf{46.1} & \underline{34.3}      & \underline{21.8}              \\
    DCPOC \cite{dcpoc}         & 33.1          & 27.5                  & 20.5              \\
    \midrule
    \textit{Co$^2$In} (ours)            & 42.5          & \textbf{37.5}         & \textbf{22.5}      \\ 
    \bottomrule
  \end{tabular}
\end{table}

\subsection{Study on Number of Tasks}

We first assess the Class-IL performance of \textit{Co$^2$In} across different task sequence lengths. Following Always Be Dreaming (ABD) \cite{abd}, we conduct experiments on CIFAR-100 using sequences of 5, 10, and 20 tasks. 
As depicted in Table~\ref{tab:tasks_classil}, both PASS and CRHD demonstrate better performance than \textit{Co$^2$In} for the 5-task setting. Nevertheless, in longer task sequence (10 and 20) settings, \textit{Co$^2$In} surpasses all the baselines. These findings prove that \textit{Co$^2$In} is effective for longer sequences of Class-IL tasks.

\begin{table}
  \caption{
    Task Incremental learning results (ACC \%) on 20-split CIFAR-100 and 25-split Tiny-ImageNet. Results with ($\dag$) stem from VDFD \cite{vdfd}.
  }
  \label{tab:tasks_taskil}
  \centering
\resizebox{0.9\columnwidth}{!}{
  \renewcommand{\arraystretch}{0.6} 
  \setlength{\tabcolsep}{5mm}{
  \begin{tabular}{l|m{2cm}<{\centering}|m{2cm}<{\centering}}
    \toprule
    \textbf{Methods}                & \textbf{20-split CIFAR-100}    & \textbf{25-split Tiny-ImageNet}     \\
    \midrule
    InstAParam$^{\dag}$\cite{InstAParam}& 51.04       & 34.64       \\
    PackNet$^{\dag}$\cite{packnet}  & 67.50           & 52.93       \\
    BLIP$^{\dag}$\cite{blip}        & 68.17           & 47.63       \\
    GEM$^{\dag}$\cite{gem}          & 68.89           & -           \\
    A-GEM$^{\dag}$\cite{a-gem}      & 61.91           & 53.32       \\
    EWC$^{\dag}$\cite{ewc}          & 71.66           & 52.33       \\
    MAS$^{\dag}$\cite{mas}          & 63.84           & 47.96       \\
    SI$^{\dag}$\cite{si}            & 59.76           & 45.27       \\
    IMM$^{\dag}$\cite{IMM}          & 57.94           & 46.81       \\
    ADNS$^{\dag}$\cite{adns}        & 77.33           & 59.77       \\
    LWF$^{\dag}$\cite{lwf}          & 74.38           & 56.57       \\
    iCaRL$^{\dag}$\cite{icarl}      & 75.75           & 56.35       \\
    PASS$^{\dag}$\cite{pass}        & 74.43           & 56.78       \\
    OWM$^{\dag}$\cite{OWM}          & 68.47           & 49.98       \\
    Adam-NSCL$^{\dag}$\cite{adam-nscl}& 75.95         & 58.28       \\
    VDFD$^{\dag}$\cite{vdfd}        & \underline{85.84} & 65.99       \\
    GPM\cite{gpm}                   & 77.55           & 68.62       \\
    \midrule        
    BMKP (ours)                     & 85.81           & \textbf{80.01}       \\
    \textit{Co$^2$In} (ours)                 & \textbf{87.76}  & \underline{78.78}       \\
    \bottomrule
  \end{tabular}
  }}
\end{table}

Next, we delve into the Task-IL performance of \textit{Co$^2$In} with a longer sequence of tasks. Following previous works \cite{adam-nscl, vdfd}, we conduct experiments on 20-split CIFAR-100 and 25-split Tiny-ImageNet, where 20-split CIFAR-100 is constructed by splitting CIFAR-100 into 20 tasks, and 25-split Tiny-ImageNet is constructed by splitting Tiny-ImageNet into 25 tasks. 
Table~\ref{tab:tasks_taskil} presents the comparison results. As can be seen, our method achieves the highest accuracy of 87.76\% on 20-split CIFAR-100, outperforming the second-best performance by a margin of 1.92\%. As for 25-split Tiny-ImageNet, both BMKP and \textit{Co$^2$In} overwhelm other methods, with BMKP achieving the highest accuracy of 80.01\% and \textit{Co$^2$In} closely following with an accuracy of 78.78\%. These results highlight the Task-IL ability of \textit{Co$^2$In} for longer sequences of tasks.

\begin{table}[t]
\centering
\caption{Accuracy (\%) of methods with linear and quadratic classifiers.}
\label{tab:quadratic}
\resizebox{1.0\columnwidth}{!}{
\setlength{\tabcolsep}{2pt} 
\renewcommand{\arraystretch}{0.6} 
\begin{tabular}{l|c|cc|cc}
\midrule                 
 \textbf{Method}     & \textbf{Classifier}   & \multicolumn{2}{c|}{\textbf{CIFAR-100}}  & \multicolumn{2}{c}{\textbf{Tiny-ImageNet}}  \\ 
 & & Class-IL  & Task-IL & Class-IL  & Task-IL \\
\midrule    
  LWF \cite{lwf}     &  Linear          & 10.05$\pm$0.6     & 63.78$\pm$4.3     & 9.36$\pm$0.1  & 58.61$\pm$1.8    \\ 
  LWF \cite{lwf}     &  Quadratic       & FAIL              & FAIL              & FAIL          & FAIL    \\ 
\midrule    
  DER++ \cite{der}   &  Linear          & 7.42$\pm$0.5      & 62.21$\pm$2.6     & 6.77$\pm$0.1  & 60.91$\pm$1.3    \\ 
  DER++ \cite{der}   &  Quadratic       & FAIL              & FAIL              & FAIL          & FAIL    \\  
\midrule    
  GPM \cite{gpm}     &  Linear          & 8.19$\pm$0.1     & \textbf{70.93$\pm$0.9}    & 7.04$\pm$0.1     & \textbf{59.84$\pm$0.2}    \\ 
  GPM \cite{gpm}     &  Quadratic       & \textbf{17.04$\pm$1.0}    & 61.11$\pm$0.8    & \textbf{15.92$\pm$0.3}    & 53.15$\pm$0.4    \\ 
\midrule    
  FAS \cite{fas}     &  Linear          & 25.79$\pm$0.7    & 70.89$\pm$0.6    & 24.29$\pm$0.3    & \textbf{60.10$\pm$0.2}  \\ 
  FAS \cite{fas}     &  Quadratic       & \textbf{27.52$\pm$1.4}    & \textbf{72.93$\pm$0.9}    & \textbf{24.63$\pm$0.3}    & 59.18$\pm$0.3\\ 
\midrule    
  \textit{Co$^2$In} (ours) &  Linear    & 16.90$\pm$0.5    & 79.65$\pm$0.4    & 8.81$\pm$0.1     & \textbf{69.31$\pm$0.5}             \\ 
  \textit{Co$^2$In} (ours) &  Quadratic & \textbf{37.54$\pm$0.9}    & \textbf{84.21$\pm$1.0}    & \textbf{31.23$\pm$0.5}    & 68.75$\pm$0.7               \\ 
\midrule  

\end{tabular}
}
\end{table}

\begin{table}[t]
\centering
\caption{
Class-IL performance (\%) of Co$^2$In using linear vs. quadratic classifiers on CIFAR-100. ``Number of Expanded Basis'' indicates the maximum number of new basis allowed to be added in each task.
}
\label{tab:quadratic_basis}
\resizebox{1.0\columnwidth}{!}{
\setlength{\tabcolsep}{4pt} 
  \renewcommand{\arraystretch}{0.6} 
\begin{tabular}{c|cc}
\midrule                 
 \textbf{Number of Expanded Basis}   &  \textbf{Linear Classifier} & \textbf{Quadratic Classifier}    \\ 
\midrule    
 0          & 8.98          &  8.50              \\ 
 1          & 12.96         &  24.85              \\ 
 2          & 13.84         &  36.14                         \\ 
 4          & 14.17         &  36.70                    \\ 
 8          & 14.86         &  37.10                    \\ 
 No Restriction & 16.90     &  37.54                \\
\midrule  
\end{tabular}
}
\end{table}

\begin{table*}[t]
\centering
\caption{Comparison of FLOPs, memory usage, and performance (Class-IL and Task-IL accuracy) on CIFAR-100 and Tiny-ImageNet.}
\label{tab:comare_memory_and_flops}
\resizebox{1.0\textwidth}{!}{
  \renewcommand{\arraystretch}{0.6} 
\setlength{\tabcolsep}{1pt} 
\begin{tabular}{l|c|ccccc|ccccc}
\midrule                 
\textbf{Method} & \textbf{Backbone}   & \multicolumn{5}{c|}{\textbf{CIFAR-100}}  & \multicolumn{5}{c}{\textbf{Tiny-ImageNet}}  \\ 
                  &   & \multicolumn{2}{c}{\textbf{Class-IL}} & \multicolumn{2}{c}{\textbf{Task-IL}} & Memory (MB)  & \multicolumn{2}{c}{\textbf{Class-IL}} & \multicolumn{2}{c}{\textbf{Task-IL}} & Memory (MB)  \\ 
& & Accuracy (\%) & Flops & Accuracy (\%) & Flops & & Accuracy (\%) & Flops & Accuracy (\%) & Flops & \\
\midrule   
  LWF \cite{lwf}    & standard ResNet18  & 10.05 & 5.57$\times 10^8$ & 63.78 & 5.57$\times 10^8$ & 42.50       & 9.36 & 2.22$\times 10^9$ & 58.61 & 2.22$\times 10^9$ & 42.50                  \\ 
  DER++ \cite{der}  & standard ResNet18  & 36.52 & 5.57$\times 10^8$ & 74.86 & 5.57$\times 10^8$ & 48.35       & 19.38 & 2.22$\times 10^9$ & 51.91 & 2.22$\times 10^9$ & 65.94                  \\ 
\midrule    
  \textit{Co$^2$In} & standard ResNet18  & \textbf{39.25} & 3.24$\times 10^9$ & \textbf{84.78} & 3.24$\times 10^9$ & 309.32 & \textbf{33.25} &  1.19$\times 10^{10}$ & \textbf{70.08} & 1.19$\times 10^{9}$ & 270.81                \\   
  \textit{Co$^2$In} & narrower ResNet18  & 37.54 & \textbf{3.49$\times 10^8$} & 84.21 & \textbf{3.49$\times 10^7$}  & \textbf{25.53}       & 31.23 & \textbf{1.43$\times 10^9$} & 68.75 & \textbf{1.43$\times 10^8$} & \textbf{23.74}                    \\ 
\midrule  

\end{tabular}
}
\end{table*}

\begin{table}[t]
\centering
\caption{Comparison with different numbers of non-decomposed layers.}
\label{tab:layer}
\resizebox{1.0\columnwidth}{!}{
  \renewcommand{\arraystretch}{0.6} 
\setlength{\tabcolsep}{2pt} 
\begin{tabular}{c|ccccc}
\midrule                 
  & \multicolumn{5}{c}{\textbf{CIFAR-100}}    \\ 
\midrule    
Non-decomp. layers  & \multicolumn{2}{c}{Class-IL} & \multicolumn{2}{c}{Task-IL} & Memory (MB)\\ 
                   & Accuracy (\%) & Flops & Accuracy (\%) & Flops & \\
\midrule    
  All decomposed   & 37.54  & 3.49$\times 10^8$   & 84.21  &  3.49$\times 10^7$  & 25.53 \\ 
  First 5   & 35.87  & 2.62$\times 10^8$   & 83.37  &  4.34$\times 10^7$  & 25.46 \\ 
  First 10  & 35.60  & 2.13$\times 10^8$   & 81.09  &  5.03$\times 10^7$  & 23.95 \\ 
  First 15  & 31.46  & 1.45$\times 10^8$   & 75.33  &  5.53$\times 10^7$  & 21.19 \\ 
\midrule    
                & \multicolumn{5}{c}{\textbf{Tiny-ImageNet}}    \\ 
\midrule    
Non-decomp. layers  & \multicolumn{2}{c}{Class-IL} & \multicolumn{2}{c}{Task-IL} & Memory (MB)\\ 
                   & Accuracy (\%) & Flops & Accuracy (\%) & Flops & \\
\midrule    
  All decomposed   & 31.23  & 1.43$\times 10^9$   & 68.75  & 1.43$\times 10^8$  & 23.74 \\ 
  First 5   & 30.19  & 6.06$\times 10^8$   & 67.11  & 1.66$\times 10^8$  & 23.14 \\ 
  First 10  & 28.94  & 5.93$\times 10^8$   & 65.82  & 1.95$\times 10^8$  & 22.74 \\ 
  First 15  & 26.56  & 4.72$\times 10^8$   & 63.23  & 2.24$\times 10^8$  & 22.45 \\ 
\midrule  
\end{tabular}
}
\end{table}

\subsection{Study the Adaptability of Quadratic Classifier} \label{sec:exp_quadratic_classifier}

This section investigates the application of the quadratic classifier to other IL methods and explores its adaptability. We select LWF, DER++, GPM, and FAS for this analysis. As shown in Table~\ref{tab:quadratic}, when applied to low-stability methods such as LWF and DER++, the training becomes highly unstable. This instability may arise from the quadratic classifier's tendency to learn compact decision boundaries, which makes it sensitive to feature variations. Even minor fluctuations in the features can lead to significant changes in the classification logits. This characteristic, however, makes it only suitable for high-stability methods. For example, the quadratic classifier yields significant improvement for \textit{Co$^2$In}, with Class-IL performance increasing by 20.64\% on CIFAR-100 and 22.42\% on Tiny-ImageNet, respectively.

However, the performance gain for FAS is relatively modest. We attribute this discrepancy primarily to the constraint that its shared basis remains fixed after the first task, whereas Co$^2$In dynamically expands its basis throughout incremental learning. To validate this, we conducted experiments restricting basis expansion in Co$^2$In. As shown in Table~\ref{tab:quadratic_basis}, without basis expansion, the quadratic classifier yields no improvement and may even degrade performance. As the number of expanded basis increases, the benefits of the quadratic classifier become more significant. This behavior aligns with GPM, where the projection mechanism encourages learning new, orthogonal weights, akin to basis expansion. We believe this dependence on basis expansion is intuitive: new basis captures task-specific features that help distinguish between old and new tasks.

\subsection{Results of \textit{Co$^2$In} on Standard ResNet18}

In the experiments, we selected a narrower version of ResNet18 as the backbone for \textit{Co$^2$In}. This section analyzes its performance using the standard ResNet18 architecture. Table~\ref{tab:comare_memory_and_flops} presents the results of \textit{Co$^2$In} on CIFAR-100 and Tiny-ImageNet. On the standard ResNet18, \textit{Co$^2$In} achieves the highest performance but incurs greater memory and computational costs. However, it also results in higher memory and computational overhead.  We recommend using a narrower ResNet18 for \textit{Co$^2$In}, as it strikes a favorable trade-off between accuracy and computational/memory overhead.

\begin{table}[t]
\centering
\caption{Comparison of pretrained model-based IL methods.}
\label{tab:pretrain-based}
\resizebox{1.0\columnwidth}{!}{
  \renewcommand{\arraystretch}{0.6} 
\setlength{\tabcolsep}{3pt} 
\begin{tabular}{l|ccc|ccc}
\midrule
\textbf{Method} & \multicolumn{3}{c}{\textbf{CIFAR-100}}  & \multicolumn{3}{c}{\textbf{Tiny-ImageNet}}  \\
                & Class-IL & Task-IL & Memory & Class-IL & Task-IL & Memory   \\
\midrule
  ACMap~\cite{acmap}       & 89.91\% & 98.35\% & 63.28 MB & 86.01\% & 96.12\% & 63.28 MB \\
  InfoLoRA~\cite{infolora} & 85.83\% & 98.62\% & 63.28 MB & 84.62\% & 93.58\% & 63.28 MB \\
  Naive LoRA               & 83.44\% & 98.03\% & 63.28 MB & 79.30\% & 91.42\% & 63.28 MB \\
  \textit{Co$^2$In} (ours) & 85.96\% & 98.74\% & 63.98 MB & 84.38\% & 96.72\% & 60.38 MB \\
\midrule
\end{tabular}
}
\end{table}

\subsection{Study on Number of Non-Decomposed Layers}

This section investigates the overhead introduced by the bi-level memory framework. To this end, we freeze the first 5, 10, and 15 layers closest to the input after learning the first task and apply the knowledge consolidation and integration only for the remaining layers. 

The results in Table~\ref{tab:layer} reveal a clear trade-off between resource efficiency and model performance. Freezing more layers significantly reduces memory usage but also decreases performance. For instance, on CIFAR-100, freezing 15 layers reduced memory usage by approximately 17\% and FLOPs by nearly 60\%. However, a simultaneous performance decline is observed: Class-IL accuracy decreases by 6.08\%, while Task-IL accuracy drops by 8.88\%. Similar trends are observed on Tiny-ImageNet, where freezing 15 layers led to a 4.67\% reduction in Class-IL accuracy.

We also observe an interesting yet reasonable phenomenon: as the number of non-decomposed layers decreases, FLOPs for Class-IL increase, while those for Task-IL decrease. This is because Class-IL requires all extended parameters to be processed in parallel, which results in higher FLOPs with more decomposed layers. In contrast, Task-IL only activates certain parts of the parameters, so a greater number of decomposed layers benefits from the accelerated knowledge integration, i.e., the reduction in FLOPs.

\section{Extension to Pretrained Model-based Incremental Learning}
\label{sec:pretrained-cl}

Although this paper focuses on training from scratch, the proposed \textit{Co$^2$In} can be seamlessly extended to pretrained model-based incremental learning. In this setting, the pretrained model provides a strong and stable feature prior. Therefore, we freeze the pretrained backbone and learn a residual branch for each adaptable layer. For the $l$-th layer at task $t$, the effective weight is formulated as
\begin{equation*}
    \bm{W}_t^l = \bm{W}_{pre}^l + \Delta \bm{W}_t^l,
\end{equation*}
where $\bm{W}_{pre}^l$ denotes the frozen pretrained weight, and $\Delta \bm{W}_t^l$ denotes the task-specific residual to be learned.

The residual branch follows the same knowledge encoding, consolidation, and integration pipeline as \textit{Co$^2$In}. Specifically, the working memory initializes $\Delta \bm{W}_t^l$ as a zero matrix with the same size as $\bm{W}_{pre}^l$ and optimizes it on the current task. After training, the learned residual is consolidated into the long-term memory by decomposing it into a pattern basis $\bm{B}_t^l$ and a compact knowledge representation $\bm{A}_t^l$. The long-term memory then integrates the basis and representations from all learned tasks into separate layers $\bm{B}^l$ and $\bm{A}^l$, enabling efficient parallel inference.


We follow ACMap \cite{acmap} and evaluate \textit{Co$^2$In} using ViT-B/16 pretrained on ImageNet as the backbone. Three representative baselines are selected: Naive LoRA, which trains an independent LoRA module for each task and merges them during inference; InfoLoRA~\cite{infolora}, a LoRA-based IL method; and ACMap~\cite{acmap}, an adapter-based IL method.
As shown in Table~\ref{tab:pretrain-based}, with comparable memory consumption, \textit{Co$^2$In} achieves the best Task-IL performance on both CIFAR-100 and Tiny-ImageNet, indicating that the proposed residual consolidation strategy effectively preserves task-specific knowledge. For Class-IL, \textit{Co$^2$In} obtains performance comparable to existing pretrained IL methods, while clearly outperforming Naive LoRA. These results show that \textit{Co$^2$In} is effective for pretrained-model-based IL.

\section{Proof of the Orthogonality in Eq.~\eqref{equation:new_pattern}} \label{sec:math}

To support the orthogonality claim in Eq.\eqref{equation:new_pattern}, we present Theorem~\ref{the:orth}, which is derived from the properties of SVD.

\begin{theorem} \label{the:orth}
    Let $\bm{B}$ be an $m \times r$ matrix with orthonormal columns, i.e., $\bm{B}^\top \bm{B} = \bm{I}_r$. Given any $m \times n$ matrix $\bm{Z}$, perform the singular value decomposition:
    \[
    \bm{U}\bm{\Sigma}\bm{V}^\top = \bm{Z} - \bm{B}\bm{B}^\top \bm{Z},
    \]
    where $\bm{U}$ is an $m \times k$ orthogonal matrix, $\bm{\Sigma}$ is a diagonal matrix containing the non-zero singular values, and $\bm{V}$ is a $k \times n$ orthogonal matrix. Then, the columns of $\bm{U}$ lie in the left null space of $\bm{B}$, i.e.,
    \[
    \bm{B}^\top \bm{U} = \bm{0}.
    \]
\end{theorem}

\begin{proof}

\begin{equation*}
  \begin{aligned}
      \bm{B}^\top \bm{U} &= \bm{B}^\top \bm{U} \bm{\Sigma} \bm{\Sigma}^{-1} 
      = \bm{B}^\top \bm{U} \bm{\Sigma}\bm{V}^\top \bm{V} \bm{\Sigma}^{-1} \\
      &= \bm{B}^\top \left(\bm{Z} \!-\! {\bm{B}}{\bm{B}}^\top \ \bm{Z}\right) \bm{V} \bm{\Sigma}^{-1}
      \!=\! \left(\bm{B}^\top\bm{Z} \!-\! {\bm{B}}^\top \bm{Z}\right) \bm{V} \bm{\Sigma}^{-1} \\
      &= \bm{0} \bm{V} \bm{\Sigma}^{-1} = \bm{0} .
  \end{aligned}
\end{equation*}

\end{proof}

\section{Conclusion} \label{sec:conclusion}

This work draws inspiration from the multi-level human memory system and proposes a bi-level memory framework for incremental learning. The bi-level memory framework consists of a working memory for acquiring new knowledge and a long-term memory for preserving previously learned information. With the separate memory units for knowledge acquiring and remembering, the bi-level memory framework ensures both high plasticity and stability. Furthermore, in light of the human memory mechanism, we devise a series of strategies for knowledge encoding, consolidation, and integration within the bi-level memory framework, which together constitute our \textit{Co$^2$In}. Extensive experiments demonstrate that \textit{Co$^2$In} effectively handles Task-IL and Class-IL with efficient memory and computation utilization.

Nonetheless, \textit{Co$^2$In} has some limitations, such as missing backward knowledge transfer and escalating memory consumption proportional to the number of tasks learned. In the future, we plan to improve \textit{Co$^2$In} by integrating advanced pattern basis editing methods, such as meta-learning. These enhancements are expected to facilitate backward knowledge transfer and effectively curb the expansion of the pattern basis, further advancing the efficiency of \textit{Co$^2$In}.

\section*{Acknowledgments}
This work was supported in part by the National Natural Science Foundation of China under Grant 62276019, 62306028, 62501043, U22B2004, in part by the Fundamental Research Funds for the Central Universities under Grant 2022JBQY007, 2026XKRC005, in part by the Beijing Natural Science Foundation under Grant L231019, in part by the RIE2025 Industry Alignment Fund – Industry Collaboration Projects (IAF-ICP) (Award I2301E0026), administered by A*STAR, supported by Alibaba Group and NTU Singapore through Alibaba-NTU Global e-Sustainability CorpLab (ANGEL), partially funded by National Research Foundation Fellowship (NRFF13-2021-0006), Singapore.

 





{\small
\bibliographystyle{IEEEtran}
\bibliography{egbib.bib}
}

\def \bioGapReduce{0pt}

\begin{IEEEbiography}[{\includegraphics[width=1.0in,height=1.2in,clip,keepaspectratio]{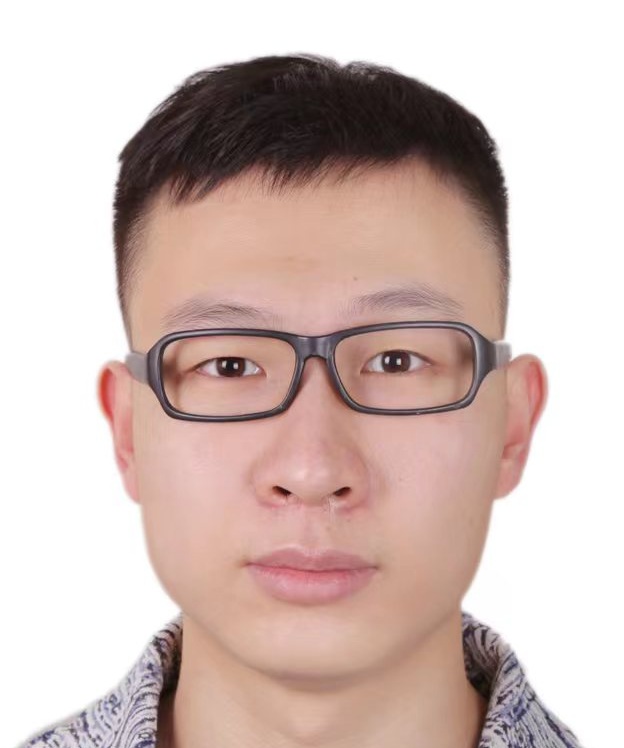}}]{Wenju Sun}
received the B.Sc. degree in Software Engineering from Beijing Jiaotong University, Beijing, China. He is currently pursuing his Ph.D. degree at the School of Computer Science and Technology, Beijing Jiaotong University, Beijing, China. His research interests include computer vision and incremental learning.
\end{IEEEbiography}

\vspace{-33pt}
\begin{IEEEbiography}[{\includegraphics[width=1.0in,height=1.2in,clip,keepaspectratio]{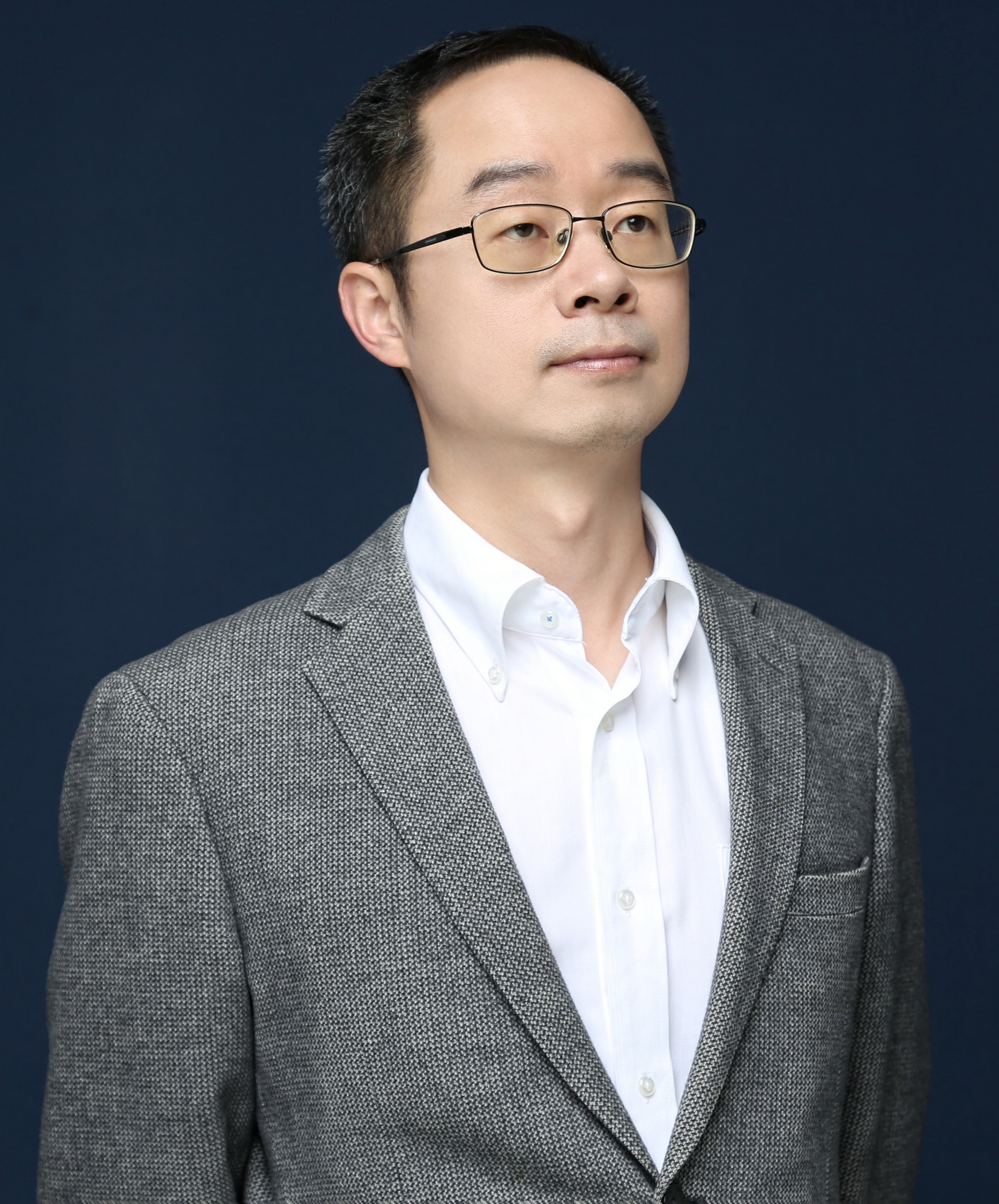}}]{Qingyong Li}
received the B.Sc. degree in Computer Science and Technology from Wuhan University, Wuhan, China, in 2001, and his Ph.D. degree in Computer Science and Technology from the Institute of Computing Technology (ICT), Chinese Academy of Sciences (CAS), Beijing, China, in 2006. He is currently a Professor at Beijing Jiaotong University, Beijing. His current research interests include computer vision and artificial intelligence.
\end{IEEEbiography}

\vspace{-33pt}
\begin{IEEEbiography}[{\includegraphics[width=1.0in,height=1.2in,clip,keepaspectratio]{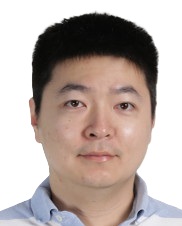}}]{Boyang Li}
received the B.Sc. degree in Computer Engineering from Nanyang Technological University, Singapore, in 2008, and his Ph.D. degree in Computer Science from the Georgia Institute of Technology, Georgia, US, in 2014. He is currently an Associate Professor at Nanyang Technological University, Singapore. His current research interests include computational narrative intelligence, multi-modal learning, and machine learning.
\end{IEEEbiography}

\vspace{-33pt}
\begin{IEEEbiography}[{\includegraphics[width=1.0in,height=1.2in,clip,keepaspectratio]{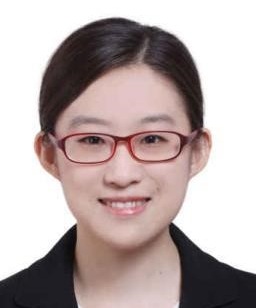}}]{Wen Wang}
received the Ph.D. degree in Computer Science and Technology from the Institute of Computing Technology (ICT), Chinese Academy of Sciences (CAS), Beijing, China, in June 2017. Currently, she is an Associate Professor at Beijing Jiaotong University, Beijing. Her research interests mainly include computer vision, pattern recognition, and machine learning.
\end{IEEEbiography}

\vspace{-33pt}
\begin{IEEEbiography}[{\includegraphics[width=1.0in,height=1.2in,clip,keepaspectratio]{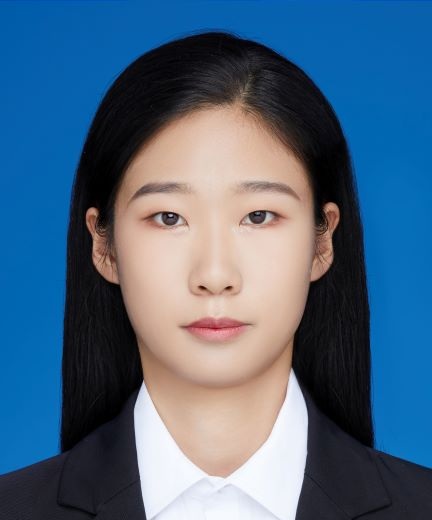}}]{Siyu Zhang}
received the B.Sc. degree in Engineering Mechanics from Beijing Jiaotong University, Beijing, China. She is currently pursuing her M.S. degree at the School of Computer Science and Technology, Beijing Jiaotong University, Beijing, China. Her research interests include incremental learning and novel class discovery.
\end{IEEEbiography}

\vspace{-33pt}
\begin{IEEEbiography}[{\includegraphics[width=1.0in,height=1.2in,clip,keepaspectratio]{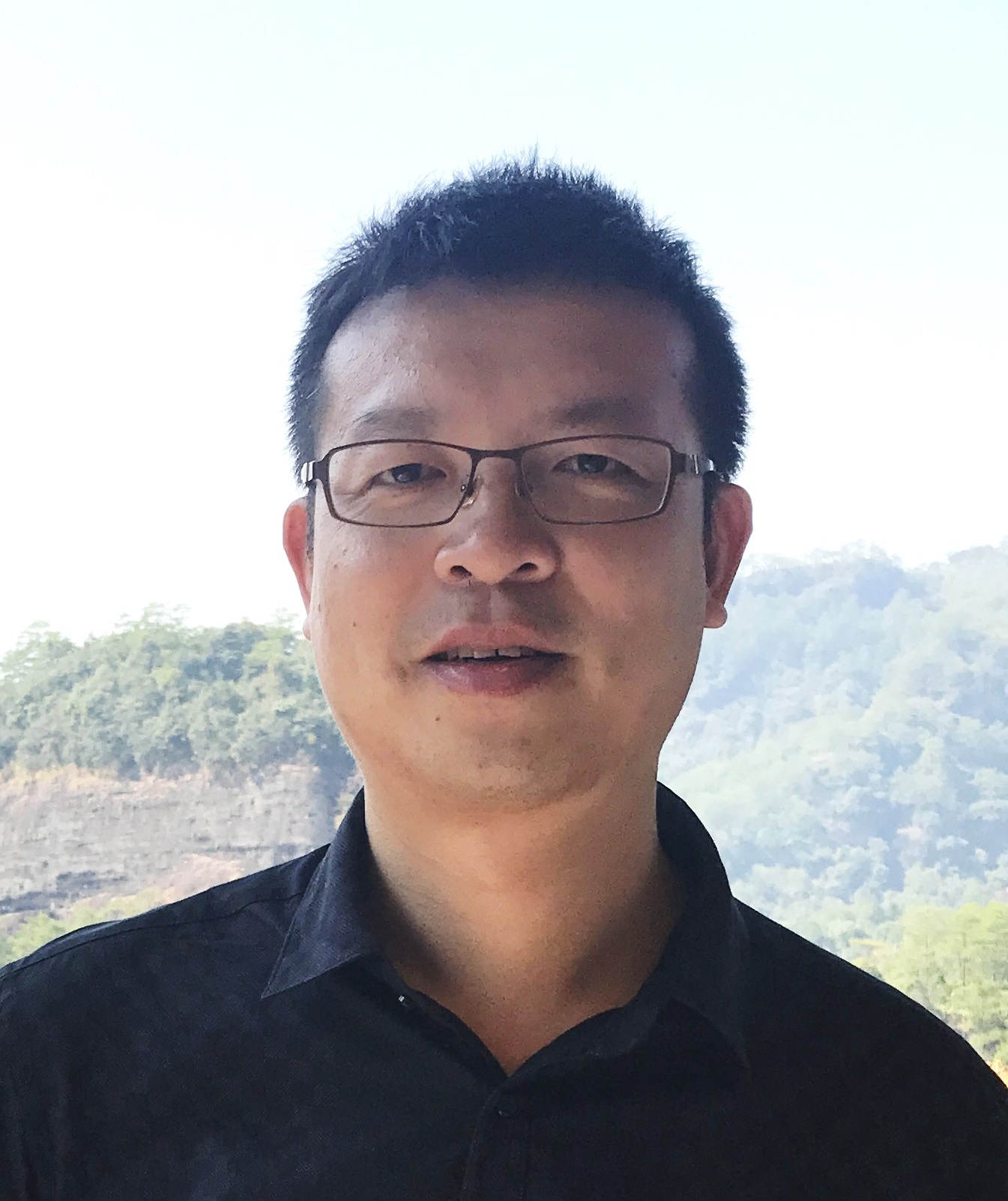}}]{Yi Wei}
received the Ph.D. degree in Computer Science and Technology from the Institute of Computing Technology (ICT), Chinese Academy of Sciences (CAS), Beijing, China, in 2008. He is currently an Associate Professor at the School of Mathematics and Computer Science, Wuyi University, Fujian, China. His research interests include causal discovery, causal inference, causal learning, and their applications.
\end{IEEEbiography}

\vspace{-33pt}
\begin{IEEEbiography}[{\includegraphics[width=1.0in,height=1.2in,clip]{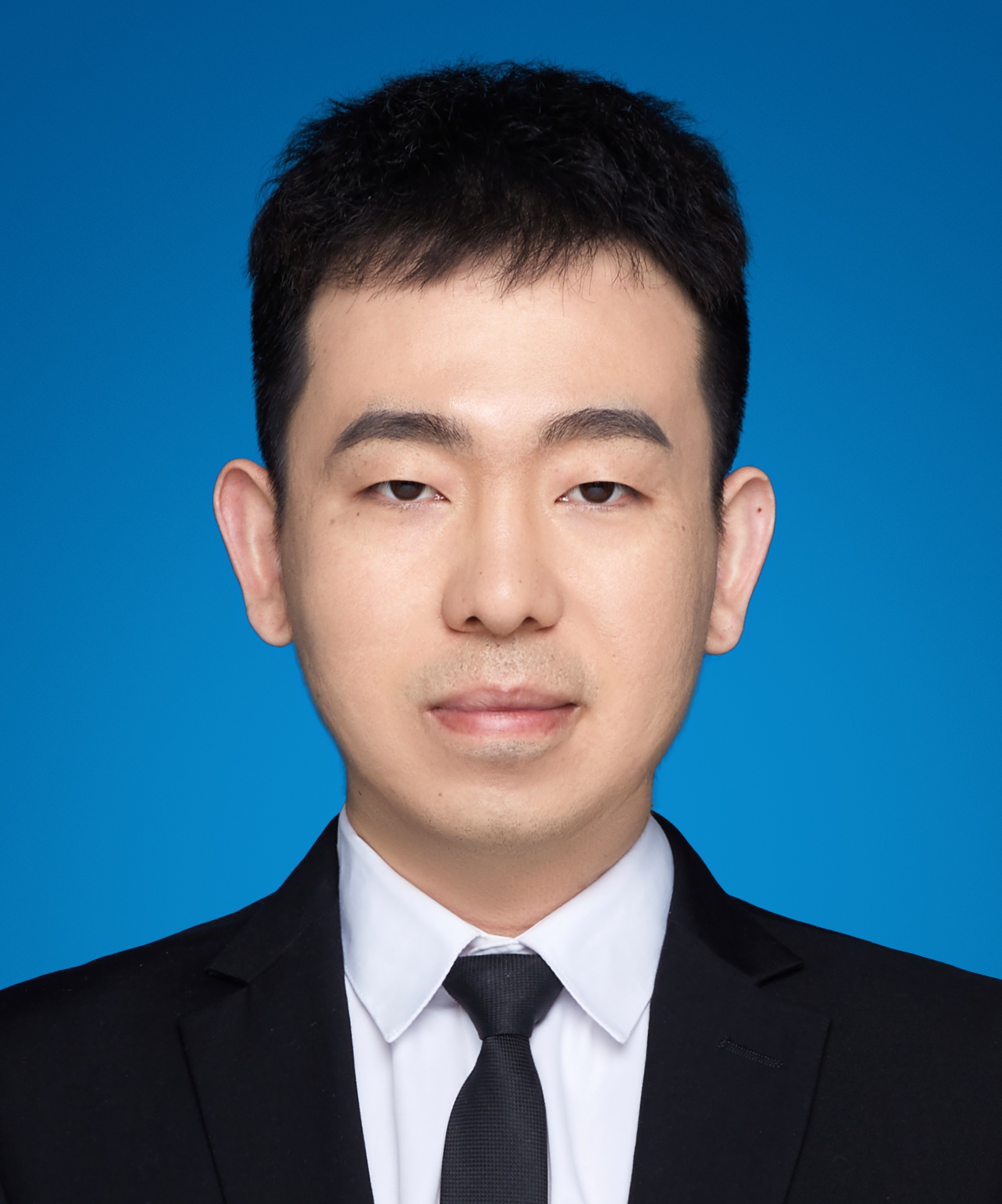}}]{Yangli-ao Geng} received his Ph.D. in Computer Science and Technology from Beijing Jiaotong University, Beijing, China, in 2021. He is currently an Assistant Professor at the School of Computer and Information Technology, Beijing Jiaotong University. His research interests include graph neural networks, spectral graph methods, and graph learning.
\end{IEEEbiography}


 




\vfill

\end{document}